%% file: main.tex
\documentclass{article}

\PassOptionsToPackage{numbers}{natbib}
\usepackage[preprint]{neurips_2025}

\usepackage[utf8]{inputenc} 
\usepackage[T1]{fontenc}    
\usepackage{hyperref}       
\usepackage{url}            
\usepackage{booktabs}       
\usepackage{amsfonts}       
\usepackage{nicefrac}       
\usepackage{microtype}      
\usepackage{xcolor}         

\usepackage{algorithm}
\usepackage[noend]{algpseudocode}

\usepackage{amsmath}
\usepackage{amssymb}
\usepackage{mathtools}
\usepackage{amsthm}

\usepackage{wrapfig}
\usepackage{subcaption}
\usepackage{graphicx}   

\usepackage[capitalize,noabbrev]{cleveref}

\theoremstyle{plain}
\newtheorem{theorem}{Theorem}[section]
\newtheorem{proposition}[theorem]{Proposition}
\newtheorem{lemma}[theorem]{Lemma}

\theoremstyle{definition}
\newtheorem{definition}[theorem]{Definition}

\theoremstyle{remark}

\usepackage[textsize=tiny]{todonotes}

\usepackage{dsfont}
\usepackage{commands}
\graphicspath{../img/}

\title{Efficient Risk-sensitive Planning via Entropic Risk Measures}

\author{%
  Alexandre Marthe \\
  ENS de Lyon\\
  Lyon, France \\
  \thanks{alexandre.marthe@ens-lyon.fr} \\
  \And
  Samuel Bounan \\
  ENS de Lyon \\
  Lyon, France \\
  \AND
  Aurélien Garivier \\
  ENS de Lyon \\
  Lyon, France \\
  \And
  Claire Vernade \\
  University of Tübingen \\
  Tübingen, Germany \\
}

\begin{document}

\maketitle

\input{core_text.tex}

\begin{ack}
C. Vernade is funded by the Deutsche Forschungsgemeinschaft (DFG) under both the project 468806714 of the Emmy Noether Program and under Germany’s Excellence Strategy – EXC number 2064/1 – Project number 390727645. CV also thanks the international Max Planck Research School for Intelligent Systems (IMPRS-IS). Aurélien Garivier and Alexandre Marthe thank the Chaire SeqALO (ANR-20-CHIA-0020-01) and PEPR IA project FOUNDRY (ANR-23-PEIA-0003).
\end{ack}

\bibliography{bibliography,bib_neurips}
\bibliographystyle{abbrvnat}

\newpage
\appendix
\onecolumn

\input{appendix.tex}

\end{document}

%% file: core_text.tex
\begin{abstract}
    Risk-sensitive planning aims to identify policies maximizing some tail-focused metrics in Markov Decision Processes (MDPs), such as (Conditional) Values at Risk. 
    In general, such optimization problems can be very costly as it is known that only Entropic Risk Measures (EntRM) can be efficiently optimized through dynamic programming.
    We show that EntRM can serve as an approximation of other metrics of interest and we propose a dual optimization problem that requires to compute the set of optimal policies for EntRM across all parameter values. 
    We prove that this \emph{optimality front} can be computed effectively thanks to a novel structural analysis of the smoothness properties of entropic risks. 
    Empirical results demonstrate that our approach achieves strong performance in risk-sensitive  decision-making scenarios.
\end{abstract}

\section{Introduction}
\label{sec:intro}

Markov Decision Processes (MDPs) \citep{puterman2014markov} 
capture sequential decision making in domains as diverse as robotics, finance, healthcare, and operations research \citep{sutton2018reinforcement,silver2017mastering, charpentier2021reinforcement,polydoros2017survey}. Standard planning problems are about maximizing the \emph{expected} cumulative reward, which can be done efficiently via dynamic programming. However, in many high-stakes applications, such as healthcare, average performance alone is not sufficient. In fact, it may be critical to limit the probability of catastrophic outcomes or to ensure that returns remain above certain thresholds with high confidence. This need has driven research on \emph{risk-sensitive} control, which incorporates measures of uncertainty and tail behavior into the objective \citep{bauerle_more_2014,tamar2015policy}. 

One of the central challenges in risk-sensitive optimization is to identify risk criteria that are both \emph{meaningful} for real-world decision making and \emph{tractable} in an MDP context. Popular approaches often revolve around the quantile-based Value at Risk (VaR) and Conditional Value at Risk (CVaR) \citep{artzner1999coherent,rockafellar2000optimization}, which are widely used to control tail risk. 
Another common objective relies on controlling a Threshold Probability \citep{white1993minimizing}, that is the probability that total returns fall below a specified level.

Despite their practical interest and interpretability properties, these common risk metrics cannot be directly and efficiently optimized in MDPs \citep{marthe2024beyond, rowland2019statistics}.
Our work addresses precisely this gap by connecting the Threshold Probability and the (C)VaR metrics to the moment-generating function of the return of a policy. In fact, in the context of MDPs, these two optimization problems can be successfully approximated by the Entropic (Exponential) Risk Measure (EntRM) \citep{howard1972risk}, the unique functional admitting a dynamic programming decomposition \citep{ben2007old,follmer2011stochastic,marthe2024beyond}. Thus, EntRM arguably leads to the ``best possible'' extension of the risk-neutral Bellman recursion to a risk-aware MDP setting. Yet, despite this computational appeal, practical usage of EntRM can be hindered by interpretability concerns, especially around selecting  the risk-tolerance parameter $\beta$.

We propose a unifying framework for risk-sensitive planning in MDPs through the study of the EntRM. Rather than directly optimizing the latter as a proxy for other target metrics, 

We prove that optimal policies evolve in a structured manner as the risk parameter changes, which allows us to derive an efficient algorithm to compute all the optimal policies for EntRM, a set that we call \emph{the Optimality Front}. We show how this set can then be used to optimize tail-focused risk measures through the \emph{Generalized Policy Improvement} principle \citep{barreto2020fast} (see Sec.~\ref{sec:methods}). We demonstrate the consequence of this new method on the optimization of the Threshold Probability, the VaR and CVaR via an empirical study on an inventory management problem.


\section{Risk Sensitive Planning}
\label{sec:prelims}

\subsection{Risk-sensitivity in MDPs}

We consider a finite-horizon MDP $\mathcal{M} \;=\; \bigl(\states, \actions, p, r, H, p_1\bigr)$,
where $\states$ is a finite set of states, $\actions$ is a finite set of actions, $p(x'\mid x, a)$ is the probability of transitioning to state $x'$ from state $x$ when action $a$ is chosen, $r(x,a)$ is the reward function specifying the immediate reward for taking action $a$ in state $x$ (bounded in $[-1,1]$), $H$ is the finite horizon, i.e., the number of decision steps, and $p_1$ is the distribution of initial states $x_1 \sim p_1(\cdot)$.
To simplify notation, we assume all rewards and transitions are stationary, i.e. $p_t=p$ and $r_t=r$ $\forall t$, and the rewards are deterministic, but our results naturally extend to the non-stationary setting.

A policy $\pi = (\pi_1, \ldots, \pi_H)$ is a sequence of decision rules $\pi_t : \states \rightarrow \actions$ that, at each step $t \in \{1,2,\dots,H\}$, selects an action based on the current state. For simplicity, we assume policies to be deterministic, but not necessarily stationary (a stationary policy verifies $\forall t, \pi_1 = \pi_t$). The corresponding \emph{cumulative reward} (or \emph{return}) is a random variable, 
$R^{\pi} \;=\; \sum_{t=1}^H r\bigl(x_t,a_t\bigr)$, 
where $x_0 \sim p_1(\cdot)$, $a_t = \pi_t(x_t)$ and $x_{t+1} \sim p(\cdot \mid x_t, a_t)$.
We also denote $R^{\pi}_h(x,a) \;=\; \sum_{t=h}^H r\bigl(x_t,a_t\bigr)$ with $x_h = x, a_h = a$, and $R^{\pi}_h(x) \;=\; R^{\pi}_h(x,\pi(x))$.

Risk-sensitive control studies generic optimization problems
$\displaystyle{\max_{\pi} \rho (R^{\pi})}$
for a given functional $\rho$ that can capture \emph{tail events} or \emph{uncertainty aversion}. Typical choices of $\rho$ are described below.

\subsection{Risk metrics and computational limitations in MDPs}
\label{sec:expected_vs_risk}

The theory of risk measures is too rich for us to provide more than a few insights below on the most widely used examples. We define the EntRM and common risk metrics, and we discuss the computational challenges that motivate our work. 

\paragraph{Entropic Risk Measure (EntRM).}
Originally introduced by \citet{howard1972risk} for MDPs, the EntRM for a random variable $X$ and parameter $\beta \in \mathbb{R}$ is defined as
\begin{align}
\label{eq:EntRM-def}
    \mathrm{EntRM}_\beta[X] =
    \begin{cases}
        \frac{1}{\beta}\,\log\!\bigl(\mathbb{E}[e^{\beta X}]\bigr), & \text{if } \beta \neq 0, \\
        \mathbb{E}[X], & \text{if } \beta = 0.
    \end{cases}
\end{align}
In the MDP context, maximizing $\mathrm{EntRM}_\beta[R^{\pi}]$ is also equivalent to maximizing $\mathrm{sign}(\beta)\,\mathbb{E}[\,e^{\beta R^{\pi}}]$, by composition with a non-decreasing function. We call it the \emph{exponential form} of EntRM, and will use both forms throughout the paper, noting that the core optimization problem remains the same.

The parameter $\beta$ encodes risk tolerance: large positive $\beta$ values emphasize a risk-seeking attitude (amplifying high returns), while negative values induce conservativeness. 
This seemingly intuitive interpretation unfortunately does not hold the test of practice for two main reasons: the exact choice of $\beta$ for a specific level of risk-sensitivity is unclear in general, 
and EntRM is not \emph{positively homogeneous}, i.e. $\mathrm{EntRM}_\beta[cX] \neq c\,\mathrm{EntRM}_\beta[X]$ for $c \in \mathbb{R}$. In finance, this is problematic when scaling outcomes by different currencies or units.  In the following, the risk parameter is always assumed to be non-positive.

\paragraph{Value at Risk (VaR).} 
VaR represents the quantile of the distribution. At level $\alpha \in (0,1)$:
\begin{equation}
\label{eq:VaR-def}
    \mathrm{VaR}_\alpha[R^{\pi}] 
    \;=\;
    \inf \Bigl\{\, x \;\big|\; \Pr\bigl(R^{\pi} \le x\bigr) \;\ge\; \alpha \Bigr\}.
\end{equation}

Although VaR is widely used, it suffers from a lack of \emph{sub-additivity}, a desirable property for risk measures ensuring that diversification does not increase perceived risk. In the context of MDPs, VaR is particularly challenging to optimize due to its non-linearity and non-monotonicity \citep{chow2018risk}. 

\paragraph{Conditional Value at Risk (CVaR).}
The CVaR \citep{rockafellar2000optimization,follmer2011stochastic,bauerle2011markov,chow_algorithms_2014} at level $\alpha$ is given by
\begin{equation}
 \label{eq:CVaR-def}\mathrm{CVaR}_\alpha[R^{\pi}]
    \;=\;
    \frac{1}{\alpha} \int_0^\alpha \mathrm{VaR}_\gamma [R^{\pi}] d\gamma.  
\end{equation}
CVaR represents the expected return in the $\alpha$-worst fraction of outcomes, providing a more comprehensive view of the risk than VaR. What makes it particularly interesting is that it is a \emph{coherent risk measure}, meaning that it satisfies both sub-additivity and positive homogeneity, compared to the risk measure mentioned previously.

In this article, we chose the convention of VaR and CVaR representing the low tail of the distributions. In the literature, it is often defined using the high tail. This has no impact on the theory as a change of variable $X \leftarrow -X$ will transpose the results. More details can be found in \Cref{app:var_family}.

\textbf{Threshold Probability.}
Certainly the most intuitive measure of risk is the probability of falling below a user-specified threshold level $T$. Solving
$\displaystyle{
    \min_\pi \;\Pr \bigl(R^{\pi} \;\le\; T\bigr),
}$
means seeking a policy whose probability of yielding a return below $T$ is minimal. The VaR and Threshold optimization are \textbf{dual} problems. A farmer who worries about the possibility of a very poor harvest in the coming year will either ask, “What is the chance that my yield will be below 2 tons?”  (Threshold Probability viewpoint), or “With 90\% confidence, what is the size of the minimal yield I can expect?” (VaR perspective).  Lowering the probability of dropping below a certain threshold (Threshold Probability) is directly tied to choosing a quantile-based cutoff for outcomes, as $\mathrm{VaR}_\alpha[X] = T$ just means that $P(X \leq T) = \alpha$.

\textbf{Computational limitations.}
The EntRM verifies a recursive Dynamic Programming equation akin to the Bellman Equation \citep{kupper_representation_2009,follmer2011stochastic,rowland2019statistics,marthe2024beyond}. Thus, for this family of metrics, the optimization problem is tractable (See \Cref{app:entrm} for details). Yet, optimization is much more complex for the other risk measures. 
On the other hand, the optimal policies for popular risk measures such as (C)VaR can be non-Markovian\footnote{It means that the policy in state $S_t$ depends on the full history up to this state and the accumulated reward so far.} \citep{li2022quantile,hau2024dynamic}, and standard optimization strategies rely on augmenting the state space with a continuous variable that contains the cumulated reward so far, which makes the problem intractable except for very simple MDPs.
The optimization of VaR and especially CVaR have been studied extensively \citep{chow_algorithms_2014, chow2015risk,chow2018risk,achab2021robustness,bauerle2011markov}, but remains a challenging task \citep{hau2024dynamic}. 

Approximation schemes have been proposed, such as \emph{Dynamic Risk Measures} \citep{bauerle2022markov} (also called \emph{Nested} or \emph{Recursive Risk Measures}). Where the goal is to recursively optimize a pseudo Bellman objective $V_h(x) = \max_a\rho[r(x,a) + V_{h+1}(X')]$, where $X'$ denotes the (deterministic) next state given $(x,a)$. Despite its appealing computational benefits, such criterion does not optimize for $\rho$ in general. In fact, the approximation is usually quite poor (see Sec.\ref{sec:experiments}) and the actual optimized objective is not law invariant, which makes it harder to interpret.
Optimizing the Threshold Probability raises the same challenges as VaR and CVaR \citep{white1993minimizing, wu1999minimizing, kira2012threshold} but seems to have been less studied and the literature lacks approximation schemes.
In short, while the EntRM is the only risk measure that can be efficiently optimized in MDPs, it is not what people use in practice due to a lack of interpretability. We propose a novel framework to answer the question:
\emph{How can we leverage the computational properties of EntRM to optimize those more preferred measures of risk?}


\section{A unified framework for risk-sensitive optimization}
\label{sec:methods}

The Threshold Probability and VaR/CVaR objectives are all related to the tail probabilities of the return distribution. These tail probabilities can be approximated with the help of exponential moments of the distribution by Chernoff's bound:
\begin{align}\label{eq:chernoff}
    \Pr\!\bigl(X \le T\bigr) 
    \;\;\le\;\;
    \inf_{\beta \leq 0} \;\exp\bigl(-\,\beta\,T\bigr)\,\mathbb{E}\!\bigl[e^{\,\beta\,X}\bigr].
\end{align}
Exponential moments are the core of Entropic Risk Measure and we explain in this section how Inequality~\eqref{eq:chernoff} leads to proxies for the risk metrics introduced above.

\textbf{From (C)VaR to EVaR.}
Solving for the rhs of \eqref{eq:chernoff} to equal $\alpha$, \citet{ahmadi-javid_entropic_2012} introduced the \emph{Entropic Value at Risk (EVaR)} as a proxy for the VaR defined by
\[
\mathrm{EVaR}_\alpha[X] = \sup_{\beta < 0} \mathrm{EntRM}_\beta\bigl[X\bigr] - \frac{1}{\beta}\log(\alpha).
\]
EVaR has been shown to be a tight approximation of the VaR, and an even tighter one for CVaR, due to  
$\mathrm{VaR}_\alpha[X] \geq \mathrm{CVaR}_\alpha[X] \geq \mathrm{EVaR}_\alpha[X]$ \citep{ahmadi-javid_entropic_2012}. 
It is a \emph{coherent} risk measure, and its use for approximating VaR for bandit algorithms was already noted by \citet{maillard2013robust} and has received growing interest recently in MDPs \citep{ni_evar_2022,hau_entropic_2023,su2024evar}. The related proxy for VaR and CVaR is 
     $\max_\pi \mathrm{VaR}_\alpha \left[R^{\pi}\right] \geq \sup_{\beta < 0}\max_\pi \mathrm{EntRM}_\beta\bigl[R^{\pi}\bigr] - \frac{\log(\alpha)}{\beta}$, where the $\sup$ and $\max$ can be swapped as the policy space is finite. 

\textbf{A proxy for the Threshold Probability.}
Using a similar idea, we derive a proxy for the Threshold Probability:
    $\min_\pi \;\Pr \bigl(R^{\pi} \;\le\; T\bigr) \leq \min_{\beta < 0}\min_\pi e^{-\,\beta\,T}\,\mathbb{E}\bigl[e^{\beta R^{\pi}} \bigr]$.
More details can be found in \Cref{app:proxy}. To the best of our knowledge, this metric has not been studied so far, and this simple approximation scheme seems novel. 

\textbf{Quality of the approximations. }
The quality of deviations bounds used to derive the proxies above depend on the tail of the distributions and are known to be more accurate for distributions with light tails \citep{vershynin2018high}.
In MDPs, the return of a policy is more concentrated around its mean when there is a rich and bounded reward signals along the trajectory, leading to thinner tails and tighter EVaR approximations.

On the other hand, existing methods to optimize VaR, CVaR and Threshold Probability rely on dynamic programming on extended MDPs, where the state space is augmented with a continuous variable that correspond to the achievable values of the return \citep{chow_algorithms_2014, white1993minimizing}. Thus, even for small MDPs, rich reward signals may quickly increase the dimension of the state space to an extent that renders the optimization intractable. Conveniently, this is when our approximation method is relevant.


\textbf{Relaxed optimization problems. }
Having derived proxy targets, a natural idea is to attempt to optimize them \emph{instead of the true objective}. We comment on these proxy optimization problems and propose below a better and more general method that builds on these intermediate solutions. 

For a given parameter $\beta$, we denote $\pi^*_\beta$ the optimal policy for the EntRM$_\beta$ \footnote{Recall that this optimal policy can be efficiently computed by DP. }:
\begin{equation}
    \label{eq:pi_star_beta}
    \pi^*_{\beta} := \argmax_{\pi \in \Pi} \text{EntRM}_\beta[R^\pi]\, .
\end{equation}
For all objectives, as suggested by the proxy derivations above, risk-sensitive optimization problems can be reduced to finding the right parameter $\beta$: 
\begin{align}
    \label{eq:var_proxy}
    \beta^*_{\text{(C)VAR}} = \arg \sup_{\beta < 0} \mathrm{EntRM}_\beta\bigl[R^{\pi^*_\beta}\bigr] - \frac{1}{\beta}\log(\alpha), \quad \beta^*_{\text{TP}} = \arg \inf_{\beta < 0} e^{-\,\beta\,T}\,\mathbb{E}\bigl[e^{\beta R^{\pi^*_\beta}} \bigr],
\end{align}
with the corresponding $\pi^*_{\beta^*}$ as the optimal policy for the optimized metric.
 Both proxy problems are effectively optimizing the EntRM, inducing important properties of the resulting optimal policies.

\begin{proposition}[\citet{hau_entropic_2023}]
    \label{pro:proxy_policies}
    For each proxy problem in \eqref{eq:var_proxy}, there exists $\beta <0$ such that the optimal policy is also optimal for the EntRM with parameter $\beta$.
    The optimal policies are deterministic and Markovian. 
\end{proposition}

Compared to the initial problems where optimal policies are usually not Markovian, the transformation to the EntRM allows finding optimal policies that are easier to compute and implement. However, the proxies in \eqref{eq:var_proxy} now involve an optimization over a \emph{continuous} range of values of $\beta<0$, and for each value, the return of the optimal policy $\pi_\beta^*$ must be computed. While we know efficient algorithms to compute this quantity for one fixed $\beta$, optimizing over a continuous range is significantly harder and the only known approaches use discretization schemes \citep{hau_entropic_2023}. See \Cref{app:proxy} for more details.

Our main contribution is to show that this continuous search problem can be done more efficiently, and more importantly, that the set of all optimal policies $(\pi^*_\beta)_{\beta<0}$ is nicely structured and can be stored to further optimize the true objectives rather than their proxy. 


\subsection{Generalized Policy Improvement}

The key observation is that the optimal policy $\pi^*_{\beta^*}$ for the proxy target \eqref{eq:var_proxy} need not be the best one for the true objective. In general, there may be $\beta'\neq \beta^*$ such that $\pi^*_{\beta'}$ achieves better performance on the true objective.

Conveniently, optimizing the proxies already requires sweeping through all $\beta$ values and computing all the optimal policies for EntRM \eqref{eq:pi_star_beta}, $\Pi^* = \{\pi^*_\beta \, |\, \beta\leq 0 \}$, that we call the \emph{Optimality Front}. While previous work discard this set during the optimization process \citep{hau_entropic_2023}, we propose to store it, and apply the \emph{Generalized Policy Improvement} principle (GPI) proposed by \citet{barreto2020fast}. Namely, multiple policies are computed using tractable objectives (here, EntRM$_\beta$ for various $\beta$), and then the one performing best under the original, possibly intractable, risk criterion is selected. Our novel and original approach is to leverage the Optimality Front such that for any risk measure $\rho$, we compute the optimal policy as
\begin{align}
    \label{eq:generalized_policy_improvement}
    \max_{\pi \in \Pi^*} \rho\bigl(R^{\pi}\bigr) 
    \quad \text{with} \quad
    \Pi^* = \{\pi^*_\beta \mid \beta \leq 0\}
\end{align}



In the next section, we give all the elements to justify that the Optimality Front of the EntRM is indeed a good set of policies to perform GPI. Mainly, we show that it can be computed efficiently and that it is \emph{small} in general because there is a bounded number of $\beta$ values for which the optimal policy changes. Importantly, it can be computed once and reused across multiple downstream risk objectives, such as VaR, CVaR, or threshold-based criteria. Evaluating each candidate policy under the target risk measure can be done efficiently using distributional planning \citep{bellemare2023distributional}, which provides access to the full return distribution of each $\pi^*_\beta$.


\section{Structural Insights into Entropic Risk Measures}
\label{sec:structure}

Optimizing EntRM for an entire range of risk parameters requires understanding the structure of EntRM optimal policies. We now show that exploiting the regularity of the $\mathrm{EntRM}_\beta$ function leads to an efficient algorithm that computes \emph{all the optimal policies} along $\beta \in \mathbb{R}$ much more efficiently than using a grid, and with better guarantees. 
In passing, understanding  how a small perturbation of the risk parameter can influence the optimal policy is also helpful from the point of view of interpretability  and robustness \citep{bauerle2024blackwell}. More formal statements\footnote{To make the paper easier to read, we present slightly informal statements in the main text and refer curious readers to the appendix for full details.} and all the proofs are given in the appendix.

\subsection{Structure Analysis}
\begin{definition}[Optimality Front]
    \label{def:optimality_front}
    For any MDP, the \emph{optimality front} is defined as $\Gamma = (\pi_k, I_k)_k$, where $(I_k)_k$ is the partition of $\mathbb R^-$ and where $\pi_k$ is the optimal policy of EntRM for all risk tolerance parameters $\beta\in I_k$.
\end{definition}
This definition is justified by the following property, formalizing the intuition that a small perturbation of the risk parameter  typically does not change the optimal policy.

\begin{proposition}
    \label{pro:finite_action_change}
    The Optimality Front $\Gamma$ contains a finite set of policies. Each policy in $\Gamma$ is optimal on a finite union of closed intervals. In other words, the mapping $\beta \mapsto \pi^*_\beta$ is piecewise constant.
\end{proposition}
Crucially, we are able to prove lower bounds on the length of these intervals locally around a specific risk parameter $\beta$, knowing the Advantage function at this specific point.
\begin{theorem}
    \label{thm:interval_policy_change}
    Let $\beta \in \mathbb{R}$ be such that there is a unique deterministic policy $\pi^*_\beta$ optimizing $\mathrm{EntRM}_\beta$. Define the Generalized Advantage function:
    \[
       A^\pi_{h,\beta}(x,a) 
       \;=\; 
       \mathrm{EntRM}_{\beta}[R^{\pi}_h(x)] 
       \;-\; 
       \mathrm{EntRM}_{\beta}[R^{\pi}_h(x,a)].
    \]
    Then, define optimality gaps  as the smallest differences over the entire MDP:
    \begin{align*}
        \Delta =
        \begin{cases}
            \frac{|\beta|}{2}\,\min_{h,x} \min_{a \neq \pi^*_{\beta,h}(x)} \frac{1}{H-h}\, A^{\pi_\beta^*}_{h,\beta}(x,a) & \text{if } \beta \neq 0\\
            2\,\min_{h,x} \min_{a \neq \pi^*_{\beta,h}(x)} \frac{1}{(H-h)^2}\, A^{\pi_\beta^*}_{h,\beta}(x,a) & \text{if } \beta = 0.
        \end{cases}
    \end{align*}
    Then, for all $\beta' \in [\beta-\Delta, \beta + \Delta],$ the optimal policy for $\mathrm{EntRM}_{\beta'}$ remains $\pi^*_\beta$.
\end{theorem}
The first bound is relevant when the risk parameter is not too small, as it scales with $\beta$. For large values of $\beta$, it balances the small optimality gap (remember that $\mathrm{EntRM}_\beta{[R^{\pi}]} \rightarrow \mathrm{ess} \inf R^{\pi}$ when $\beta \rightarrow -\infty$, so the gaps tend to $0$). The degeneracy at $\beta=0$ is circumvented by the second bound.

Knowing this structure of intervals, the only information we need to determine the optimality front is the location of these subinterval boundaries. We call \emph{breakpoints} these risk-parameter values at which the optimal policy changes, and we show that the resulting change is generally only local (see \Cref{app:policy_change} for the proof and more details).
\begin{proposition}
    \label{pro:policy_change}
    Consider a random MDP with reward functions and transitions $\left(r_t(x,a)\right)_{t,x,a}$ and $\left(p_t(x)\right)_{t,x}$ generated from, say, independent uniform distributions. With probability $1$, for each breakpoint $\beta \in \breakpoints$ there is a single state-horizon pair for which the optimal action changes: if $\pi^1$ is optimal for $\beta \in [\beta_1, \beta_2]$ and $\pi^2$ is optimal for $\beta \in [\beta_2, \beta_3]$, then there exists a unique state $x$ and time step $t$ such that $\pi^1_t(x) \neq \pi^2_t(x)$. 
\end{proposition}

\subsection{Computing the Optimality Front}

The first step towards computing the Optimality Front is to \emph{find the breakpoints}. We first show that a direct approach is not feasible but instead we can exploit \Cref{thm:interval_policy_change}. Then, we present our algorithm, \emph{Distributional Optimality Front Iteration} (\algoname~), based on efficient (distributional) value iteration \citep{bellemare2023distributional}. 

\textbf{Finding the Breakpoints. } Breakpoints mark the transition between two different intervals of optimality. According to \Cref{pro:finite_action_change} there must be at least two optimal policies for those particular parameter values.
This yields a system of equations characterizing the breakpoints.

\begin{proposition}
\label{pro:eq_breakpoint}
    Assume $\pi^1$ and $\pi^2$ are such that $\pi^1$ is optimal for $\beta \in [\beta_1, \beta_b]$ and $\pi^2$ is optimal for $\beta \in [\beta_b, \beta_2]$. Then $\beta_b$ satisfies:
    \[
        \forall h,x, \quad 
        \mathrm{EntRM}_{\beta_b}[R^{\pi^1}_h(x)] 
        \;=\;
        \mathrm{EntRM}_{\beta_b}[R^{\pi^2}_h(x)]
    \]
\end{proposition}
Note that since the policies only differ in one state-horizon pair (Prop.~\ref{pro:policy_change}), most of these equations are trivial. Nevertheless, one could attempt to use them to directly compute the breakpoints by resolving a system of equations. Unfortunately, we argue in \Cref{app:exact_breakpoints} that this is unfeasible due to the lack of regularity of the EntRM functions and to the computational complexity of the problem. 
In general, we show that the lack of “regularity” in these functions makes it impossible in practice to know in advance how many breakpoints, or optimal policies, might appear between two given points.

However, the advantage of using the EntRM in MDPs is that the optimal policies can be computed recursively by Dynamic Programming and we show that this implies a structure on the breakpoints. 

\begin{proposition}
    \label{pro:recursive_breakpoint}
    (Informal) Let $\breakpoints^h$
    be the set of breakpoints at time $h$ and 
    $\breakpoints^h(x)$ 
    the corresponding set of breakpoints when starting at state $x$. 
    Then, the sets verify the recursive equation
    \[
    \breakpoints^h = \breakpoints^{h+1} \cup \left( \bigcup_{x\in\states} \breakpoints^h(x) \right)
    \]  
\end{proposition}
A formal statement and the proof can be found in \Cref{app:recursive_breakpoint}.
This result shows that the set of breakpoints is simply a union over the per-state breakpoints, and they can be computed via a backward recursion. 

We now have all the tools to build an incremental approach that we call \emph{FindBreaks} and a detailed pseudo-code is in \Cref{app:single_state} but we give the intuition here. 
\Cref{thm:interval_policy_change} tells us that knowing the Entropic risk at a given point allows us to identify an interval of $\beta$ values over which the optimal policy does not change.
This fact can be utilized to `jump' over $\beta$ values.
The process is the following. At a given state, assume the distribution of the return for each action is known $(\eta(x,a))_{a\in\actions}$. Start with $\beta=0$ and iterate the following steps:
\begin{enumerate}
    \item Evaluate the Generalized Advantage function and the optimality gaps (see \Cref{thm:interval_policy_change}),
    \item Use the optimality gaps to get a lower bound $\beta-\Delta$ on the next breakpoint, and `jump': $\beta\gets \beta-\Delta$
\end{enumerate}

At some point, when getting close to a breakpoint, the increments $\Delta$ will get closer to 0. Then use a minimal increment $\epsilon$ until the optimal action changes. This can be done in parallel over states thanks to \Cref{pro:recursive_breakpoint}. 

This may not be the optimal way to compute the breakpoints but it exploits all the structure of the problem: both the regularity of the exponential functions and the recursive properties of the MDP optimization allow to speed up the process. The general question of characterizing optimality for this problem is a challenging open problem we leave for future work. Our final risk-sensitive optimization algorithm below is fully modular and could integrate any other breakpoint-search algorithm. 

\textbf{Distributional Optimality Front Iteration. }
Combining both insights from Dynamic Programming and the approximation of Optimality Intervals, we derive an algorithm to compute the Optimality Front up to a desired accuracy.

This algorithm keeps in memory the distribution of the return recursively. While not compulsory, it accelerates the computation of several values of the EntRM with same reward distribution. For more details on the Dynamic Programming computation of return distributions, see \citet{bellemare2023distributional}.


            
            



\begin{algorithm}
  \caption{\algoname - \textbf{D}istributional \textbf{O}ptima\textbf{l}ity \textbf{F}ront \textbf{I}teratio\textbf{n}}
  \label{alg:mdp}
  \begin{algorithmic}[1]
    \Require Precision $\varepsilon \in (0,1)$; MDP $\mdp(\states, \actions, \transitions, \rewards, \horizon)$ parameters.
    \State Select lower bound $\beta_{\min}$ \Comment{Computed or handpicked}
    \State $\mathcal{I}_H \gets [\beta_{\min}, 0]$ \Comment{Starting interval}
    \State $\nu_H(x) \gets \delta_0$ \Comment{\small Optimal return distribution at timestep $H$}
    \For{$h \gets H$ \textbf{to} $1$}
      \For{$x \in \states$}
        \For{$I \in \mathcal{I}_h$}
          \State $\eta^I_h(x,a) \gets \varrho(x,a) * \sum_{x'} \transitions(x'\mid x,a)\,\nu^I_{h+1}(x')$ \Comment{\small Return distributions}
          \State $\{\mathcal{J}, (a^*_j)_{j\in \mathcal{J}}\} \gets \textsc{FindBreaks}\bigl(\epsilon,\;(\eta^I_h(x,a))_{a}, I\bigr)$ \Comment{\small Apply \Cref{alg:state} on $(\eta^I_h(x,a))_{a}$ as the reward distribution for each action}
          \For{$j \in \mathcal{J}$}
            \State Add $j$ to $\mathcal{I}_{h-1}$ \Comment{\small Update intervals for next timestep}
            \State $\nu^{j}_{h}(x) \gets \eta^I_h\bigl(x,a^*_j\bigr)$ \Comment{\small Store optimal return distribution}
          \EndFor
        \EndFor
      \EndFor
    \EndFor
    \State \textbf{Output} $\Gamma = \left( \pi_k, I_k \right)_k, (\eta_0^{k})^k$ \Comment{\small Optimality Front, distributions}
  \end{algorithmic}
\end{algorithm}

\algoname~outputs the Optimality Front $\Gamma$ as well as the return distributions of each optimal policy and it remains to solve \eqref{eq:generalized_policy_improvement}: $\min_k \rho(\pi_k)$, following the Generalized Policy Improvement principle discussed above. In the experimental section below, we simply call this combined optimization the \emph{Optimality Front} method.

\textbf{About the computational cost.}
Calling $B$ the number of breakpoints in the optimality front of the MDP, the number of calls to FindBreaks is bounded by $\Theta(|\states|HB)$ and thus heavily depend on the number of optimal policies. For a low number, only a few calls will be made and only a few Q-value evaluations will have to be computed. Using the empirical observations on the number of breakpoints, the number of calls to can be estimated to be in the order of $(|\states|H)^2$. The total complexity ultimately depends on the complexity of FindBreaks. An empirical study of our implementation can be found in \Cref{app:single_state}. In practice we observe that FindBreaks runs in $O(|A|f(1/\varepsilon))$ time, with some sublinear $f$.
When the support of the accumulated reward is too large, the implementation of the  distributional induction can benefit from the approximation schemes described in \cite{bellemare2023distributional}.


\section{Numerical Experiments}
\label{sec:experiments}

This section evaluates the effectiveness and efficiency of the Optimality Front approach on simple risk-sensitive planning tasks. Our goal is primarily to demonstrate the performance of our method compared to existing baselines mentioned earlier, namely the EVaR optimization approach of \citet{hau_entropic_2023} (\emph{Proxy Optimization} thereafter) and the \emph{Nested Risk Measure} \citep{bauerle2022markov} implementing an approximate (invalid) value iteration algorithm.  All the experiments in this section are of relative small scale and run on a standard laptop in a few seconds. 


\textbf{Environment. }
 We test our method on two settings, the Inventory Management MDP \citep{bellman1955optimal,scarf1960optimality}, which is a standard model for an important logistics problem, and the Cliff environment, that we mainly use to visualize the optimality front test the empirical efficiency of \emph{Find Breaks} compared to a naive grid-based discretization. 
 
 In the Inventory Management problem, the goal is to maximize the profit of a store selling one extensive good. The store has a strict maximal capacity of $M=10$. At each time step, the state of the store is its number of available goods, $x_t \in [M]$, and it can buy (action) a quantity $a_t\in [M]$ of new goods. The reward obtained is the profit minus the costs: 
 $r_t = [f(D_t,x_t,a_t) - C_m(x_t) - C_c(a_t)]/4M$, where $D_t$ is the random demand modeled by a binomial $D_t \sim B(0.5,M)$, $f(D_t,x_t,a_t) =4 \min(D_t,x_t+a_t)$ is the sales profit, $C_m(x_t)=1x_t$ is the maintenance cost, and $C_c(a_t)=3+2a_t$ is the order cost. We considered a horizon $H=10$ with $s_0=0$.
Optimal policies in the Inventory Management MDP can be parametrized by two thresholds $(s_t, S_t)$ \citep{scarf1960optimality}: at time $t$, if the stock $x_t$ is less than $s_t$, then agent should buy goods so that they have a stock of exactly $S_t$, i.e. $a^*_t = S_t-x_t$. 

\textbf{Optimality Front} 
For our method, \emph{Optimality Front}, we executed \algoname~ only once, with accuracy $\varepsilon = 10^{-2}$ on the chosen environment. It outputs a set of return distribution for all EntRM-optimal policies. We then computed all the metrics following \Cref{eq:generalized_policy_improvement} by applying the functional of choice on all returned optimal distributions and selecting the optimal value. For the Inventory Management problem described above, ,with our choice of rewards and costs, the Optimality Front contained $18$ different optimal policies.

\subsection{Threshold Probability}
\begin{table}[tbp]
  \centering
  \small 
  \begin{subtable}[t]{0.48\textwidth}
    \centering
    \caption{Evaluation of $P(R^\pi \le T) ~~(\downarrow)$}
    \label{tab:tau_results}
    \begin{tabular}{lcc}
      \toprule
      $T/\mu^*$    & 0.25  & 0.33  \\
      \midrule
      \textbf{Optimality Front} & $\mathbf{1.26e^{-5}}$  & $\mathbf{8.40e^{-5}}$  \\
      Proxy Optimization        & $2.33e^{-5}$           & $1.18e^{-4}$        \\
      Risk neutral optimal     & $1.11e^{-4}$           & $4.24e^{-4}$         \\
      Nested Risk Meas.     & $1.54e^{-3}$           & $8.37e^{-3}$            \\
      \midrule
      Optimal value            & $6.29e^{-7}$           & $8.22e^{-6}$       \\
      \bottomrule
    \end{tabular}
  \end{subtable}
  \hfill
  \begin{subtable}[t]{0.48\textwidth}
    \centering
    \caption{Evaluation of $\mathrm{(C)VaR}_\alpha[R^\pi] ~~(\uparrow)$}
    \label{tab:var_results}
    \begin{tabular}{lcccc}
      \toprule
      Risk Measure                 & \multicolumn{2}{c}{VaR} & \multicolumn{2}{c}{CVaR} \\
      Risk parameter $\alpha$      & 0.05 & 0.1 & 0.05   & 0.1    \\
      \midrule
      \textbf{Optimality Front}    & $\mathbf{1.25}$ & $\mathbf{1.33}$ & $\mathbf{1.14}$ & $\mathbf{1.21}$ \\
      Proxy Optimization           & $1.22$          & $1.30$          & $1.13$          & $1.20$          \\
      Risk neutral optimal         & $1.22$          & $\mathbf{1.33}$ & $1.11$          & $1.19$          \\
      Nested Risk Measure          & $0.88$          & $0.95$          & $0.75$          & $0.84$          \\
      \bottomrule
    \end{tabular}
  \end{subtable}
  \caption{Comparison of our two evaluation metrics under different policies.}
  \label{tab:combined_results}
  \vspace{-.5cm}
\end{table}

We first optimize policies to minimize the probability that the return falls below a threshold $T<\mu^*$ where $\mu^*$ is the optimal mean return. For this problem, we were able to compute the non-Markovian optimal policies via Dynamic Programming on the augmented state space. On Table~\ref{tab:combined_results}~(a),  we report the estimated value of the probability of falling below the imposed threshold. 
We observe that using the Optimality Front method outperforms the risk-neutral optimal policy by up to a factor $10$. On the other hand, \emph{Nested Risk Measure} fails to find a good policy and performs worse than the risk-neutral one. While \emph{Proxy Optimization} achieves reasonable results, this experiment shows that Generalized Policy Optimization on the Optimality Front performs better. 

The true optimal value here is significantly better than what any Markovian policy can achieve, especially for low thresholds. This is due to the density and high randomness of the reward that strongly affect the distance to the threshold at each step.  
In goal-oriented MDPs, with scarce reward, this gap mostly vanishes (ex. see Cliff in \Cref{app:more_experiments}).

\subsection{Value at Risk family}

In this second experiment, we maximize lower quantiles of the return, aiming for policies with thin tails again, but through a different optimization objective.
\citet{hau_entropic_2023} already showed that optimizing over the EVaR performed better than previous methods to optimize VaR or CVaR, including using augmented MDPs, so we do not compute the optimal augmented-state DP policies for this second experiment and directly compare to \emph{Proxy optimization}. We observe in \Cref{tab:var_results} that for both VaR and CVaR, \algoname~achieves slightly better or comparable performance to the state-of-the-art \emph{Proxy Optimization}.

\subsection{Optimality Front and sample efficiency}

\begin{figure}[ht]
    \centering
    \includegraphics[width=0.45\textwidth]{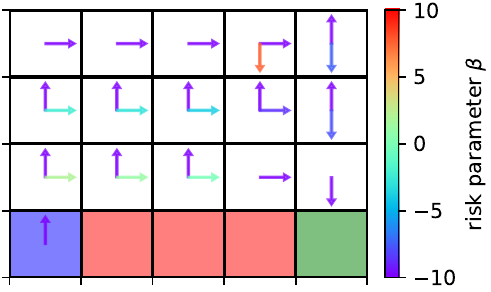}
    \hspace{0.3cm}
    \includegraphics[width=0.45\linewidth]{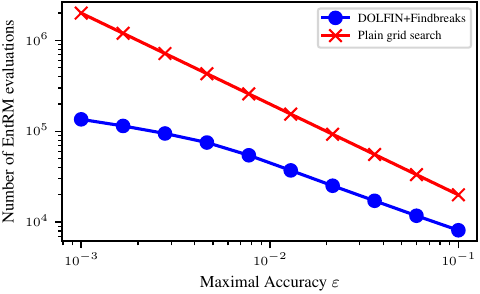}
    \caption{(Left)Illustration of all cliffs $\beta$-optimal policies. Arrows are of the color of the smallest $\beta$ value corresponding to the optimal action $[\pi^*_{\beta}]$. (Right) Number of evaluations of EntRM required to obtain an accuracy~$\varepsilon$ on the breakpoints using our method (blue) vs.\ a naive grid (red). The efficiency ratio goes up to 15.}
    \vspace{-0.3cm}
    \label{fig:cliff_policies}
\end{figure}

We illustrate the Optimality Front on the classic Cliff environment \citep{sutton2018reinforcement} (\Cref{fig:cliff_policies} Left). As the risk parameter $\beta$ grows on a continuous scale from $[-10,10]$, the optimal policy shifts on a finite number of breakpoints reflected in the color of the corresponding arrows. For each $\beta$ evaluated, we compute the optimal policy $\pi^*_\beta$ \eqref{eq:pi_star_beta} by Dynamic Programming. The Optimality Front approach relies \emph{FindBreaks} to avoid computing the optimal policy on a naive grid of $\beta$ values. It is hard to theoretically bound the number of evaluations needed to find all the breakpoints, as it strongly depends on the structure of the MDP. We illustrate the computational gains on this simple example on Fig.~\ref{fig:cliff_policies} (Right), showing that \emph{FindBreaks} can save up to an order of magnitude of evaluations.

\section{Limitations and Discussions}
\label{sec:limitations}

\textbf{Planning Setting.} This paper only considers the \emph{planning} setting, when the environment dynamics are fully known, and only small-scale problems are considered. The extension to more general settings is natural yet challenging because of the new problems raised. Optimizing the EntRM in the reinforcement learning setting has been studied before \citep{liang2024bridging}, but only for a fixed risk parameter, not a full range as studied here, so it is not clear how this can be done efficiently. Considering larger-scale environment would require using function approximations. In this case, computing the exact Optimality Front may not be tractable anymore, but the principle of Generalized Policy Improvement still holds. One idea could be to compute only a few policies associated to hand-picked risk parameters, and apply the same method. Those challenges are left for future work.

\textbf{Approximation tightness.} The approximations used in \Cref{sec:methods} are fully dependent on the return distribution and hard to control precisely in general. We are not aware of distribution-dependent measure of the tightness of EVaR, and it is unclear to us whether it is possible to prove any relevant error bounds for our \emph{Optimality Front}. 

\textbf{Complexity of the algorithm.} The theoretical complexity of \algoname~ is a function of the number of breakpoints. Those breakpoints arise from the MDP dynamics and their number cannot be simply expressed as a function of the main parameters of the MDP ($\states$, $\actions$ or $H$). The only bound we can obtain is highly pessimistic and does not allow us to provide meaningful and problem-dependent complexity results. Empirical results show a clear gain on a toy problem and these gains only get bigger with the size of the MDP (see \Cref{app:more_experiments}).

\section{Conclusion}
\label{sec:conclusion}

We propose a unified framework for optimizing risk-sensitive objectives in Markov Decision Processes. Leveraging the computational advantages of the Entropic Risk Measure (EntRM), we provide an efficient algorithm for computing the optimality front, a family of policies which are optimal on a range of risk tolerance values. This also allows us to approximate key metrics such as Threshold Probabilities, Values at Risk and Conditional Values at Risk. Our algorithm demonstrates significant practical benefits in both efficiency and policy quality. This approach not only enhances risk-sensitive planning but also provides a versatile tool for tackling a variety of decision-making problems under uncertainty. It remains to be extended beyond planning in the context of learning unknown transition probabilities and reward distributions.

%% file: appendix.tex
\section{EntRM supplements}
\label{app:entrm}

In this section, we give more insights about the Entropic Risk Measure. We first recall the definition of the EntRM.

Let $X$ be a random variable. The Entropic Risk $\mathrm{EntRM}_\beta[X]$ of parameter $\beta \in \mathbb{R}$ is defined as 
\begin{align*}
    \mathrm{EntRM}_\beta[X] =
    \begin{cases}
        \frac{1}{\beta}\,\log\!\bigl(\mathbb{E}[e^{\beta X}]\bigr), & \text{if } \beta \neq 0, \\
        \mathbb{E}[X], & \text{if } \beta = 0.
    \end{cases}
\end{align*}

\paragraph{Risk parameter interpretation.}
The parameter $\beta$ of EntRM measures risk tolerance. Large (positive) values of $\beta$ encourage a riskier behavior, as only the largest values taken by $X$ contribute significantly to the expectation. Conversely, negative values of $\beta$ promote a conservative behavior that minimizes potential losses, irrespective of the maximum reward. This behavior is illustrated by the values of the EntRM when the risk parameter approaches $\pm \infty$.

If the support of $X$ is $(R_{\min}, R_{\max})$, then 
\[
\lim_{\beta \rightarrow -\infty} \mathrm{EntRM}_\beta[X] = R_{\min} \text{ and }\lim_{\beta \rightarrow +\infty} \mathrm{EntRM}_\beta[X] = R_{\max}.
\]

When $\beta$ approaches zero, the EntRM converges to the expected value, leading to risk-neutral behavior: 
\[
\mathrm{EntRM}_\beta (X) \underset{\beta \to 0}{=} \mathbb E[X] + \frac{\beta}{2} \mathbb V[X] + o(\beta).
\]  
The risk neutral behavior obtained for $\beta = 0$ is just the expectation, and the variance appears as the sensitivity of EntRM. For a Gaussian random variable $X \sim \mathcal{N}(\mu, \sigma^2)$, it holds exactly that $\mathrm{EntRM}_\beta(X) = \mu + \frac{\beta}{2}\sigma^2$.

\paragraph{Dynamic programming.} Q-value function in the case of Entropic Risk Measure is defined as follows:
\begin{align}
\label{eq:entrm_bellman}
    Q^{\pi}_{h,\beta}(x,a) = \frac{1}{\beta}\log\left(\mathbb E\left[\exp\left(\beta R_h^\pi(x,a) \right) \right]\right).  
\end{align}
As shown in \citet{howard1972risk}, this Q-value also satisfies a Bellman equation and thus it can be optimized efficiently by value iteration, and can result in a deterministic optimal policy.

The Q-value function for the entropic risk of parameter $\beta$ satisfies the following optimal Bellman equations:
\begin{align}
    Q^{\pi}_{h,\beta}(x,a) &= \mathrm{EntRM}_\beta[r_h(x,a)] + \mathrm{EntRM}_\beta \left[Q_{h+1, \beta}^{\pi}(X',\pi(X'))\right]\\
    Q^{ * }_{h,\beta}(x,a) &= \mathrm{EntRM}_\beta[r_h(x,a)] + \mathrm{EntRM}_\beta \left[\max_{a'} Q_{h+1, \beta}^{*}(X',a')\right]
\end{align} 
Where $X' \sim p(\cdot | x,a)$.

\paragraph{Principle of Optimality.} Another implication of the Bellman equation is that the Entropic Risk Measure verifies the principle of Optimality, meaning there exists an optimal policy that is optimal for any subproblem: 
\begin{align}
    \exists \pi^*,\ \forall x,h, \quad \pi^*_{h:H} \in \argmax_\pi \mathrm{EntRM}_\beta[R^{\pi}(x)]
\end{align}
with $\pi_{h:H} = (\pi_h, \dots, \pi_H)$. This is in general not true for other risk measures, where a policy may only be optimal for a specific timestep and state, and can be suboptimal at intermediate timestep.

\paragraph{Exponential form}
While the EntRM is often defined as in \Cref{eq:EntRM-def}, the $\frac{1}{\beta} \ln$  serves solely as a renormalization factor. It does not influence the preference ordering between returns, and thus has no impact on the optimal policy in an MDP. For this reason, we also consider the exponential form $E_\beta[X] = \text{sign}(\beta)\mathbb{E}_\pi[\exp(\beta X)]$ (also called \emph{exponential utility}), which can appear in some applications, like for the Threshold Probability objective. If a policy $\pi$ is optimal for $\mathrm{EntRM}_\beta$, it is also optimal for $E_\beta$, and vice versa. In general:
\begin{align*} 
    \mathrm{EntRM}_\beta(X_1) > \mathrm{EntRM}_\beta(X_2) 
    \Longleftrightarrow 
    \text{sign}(\beta)\cdot \mathbb{E}[\exp(\beta X_1)] > \text{sign}(\beta)\cdot \mathbb{E}[\exp(\beta X_2)] 
\end{align*}
and thus,
\begin{align*} 
    \argmax_{\pi} \mathrm{EntRM}_\beta[R^\pi] = \argmax_{\pi} \text{sign}(\beta)\mathbb{E}_\pi[\exp(\beta R^\pi)].
\end{align*}
This equivalence is used at several occasions below.


\section{More on the Value at Risk family of risk measures}
\label{app:var_family}

In this section, we provide additional details on the \emph{Value at Risk} (VaR), \emph{Conditional Value at Risk} (CVaR), and \emph{Entropic Value at Risk} (EVaR) measures discussed in the main text. Two main points are covered: (1) a clarification of conventions and notation, and (2) the Derivation of EVaR.

\subsection{Conventions and Notation}

In the finance literature, VaR is generally introduced as the $\bigl(1 - \alpha\bigr)$-quantile of a \emph{loss} distribution, referring to a worst-case tail. It can also be defined via an $\alpha$-quantile for \emph{gains} (or returns), leading to slightly different formulas or sign conventions. For instance, one would usually write
\[
  \Pr\bigl(X \ge \mathrm{VaR}_{1-\alpha}[X]\bigr) = \alpha.
  \quad\text{while others use}\quad
  \Pr\bigl(X \le \mathrm{VaR}_\alpha[X]\bigr) = \alpha,
\]
Both definitions capture the idea of a critical quantile but from opposite sides of the distribution (gains vs.\ losses).

Throughout this work, we consistently adopt the low-tail perspective, defining
\[
  \mathrm{VaR}_\alpha[X] 
  \;=\;
  \inf\bigl\{\;x \,\big|\; \Pr(X \le x) \;\ge\;\alpha \bigr\},
\]
so that VaR at level $\alpha$ is simply the $\alpha$-quantile of $X$ from below. We focus on this definition to remain consistent with our usage of “negative outcomes” or “low returns” as the main source of risk. Should the convention change, one can substitute $X \leftarrow -X$ to convert between definitions without altering the mathematical properties.

To follow the change of convention, we also define the \emph{Conditional Value at Risk} (CVaR) as
\[
    \mathrm{CVaR}_\alpha[X]
    \;=\;
    \frac{1}{\alpha} \int_0^\alpha \mathrm{VaR}_\gamma[X] \, d\gamma.
    \quad\text{instead of}\quad
    \mathrm{CVaR}_{1-\alpha}[X]
    \;=\;
    \frac{1}{\alpha} \int_0^\alpha \mathrm{VaR}_{1-\gamma}[X] \, d\gamma.
\]
and the \emph{Entropic Value at Risk} (EVaR) as
\[
    \mathrm{EVaR}_\alpha[X]
    \;=\;
    \sup_{\beta < 0} \left\{ \frac{1}{\beta}\ln \mathbb E[e^{\beta X}] - \frac{1}{\beta}\ln(\alpha) \right\}
    \quad\text{instead of}\quad
    \inf_{\beta > 0}\left\{ \frac{1}{\beta}\ln \mathbb E[e^{\beta X}] - \frac{1}{\beta}\ln(1-\alpha) \right\}.
\]

\subsection{Entropic Value at Risk (EVaR)}

\emph{Entropic Value at Risk} (EVaR) was proposed as a tighter exponential-based bound on VaR and CVaR (\citealp{ahmadi-javid_entropic_2012}). Its derivation stems from the fact that exponential moment bounding can yield a close approximation to quantile-based measures. 

\begin{proof}
    By Chernoff's inequality, for any $\beta < 0$ we have:
    \[
      \Pr\bigl(X \leq \ell\bigr) \;\le\; \mathbb{E}\bigl[e^{\beta X}\bigr]\exp\bigl(-\beta\,\ell\bigr).
    \]
    Solving the equation
    \(
      \mathbb{E}\bigl[e^{\beta X}\bigr]\exp\bigl(-\beta\,\ell\bigr) \;=\; \alpha
    \)
    for $\ell$, we obtain
    \[
      \ell \;=\; a_X(\beta,\alpha)
      \;=\;
      \frac{1}{\beta}\,\ln\!\bigl(\mathbb{E}[e^{\beta X}]\bigr)
      \;-\;
      \frac{1}{\beta}\ln(\alpha).
    \]
    Hence, for any $\beta<0$ and $\alpha\in(0,1]$, the inequality
    \[
      \Pr\bigl(X \leq a_X(\beta, \alpha)\bigr)
      \;\le\;
      \alpha
    \]
    implies
    \[
      a_X(\beta, \alpha)
      \;\le\;
      \mathrm{VaR}_\alpha(X).
    \]
    Hence,
    \[
      \mathrm{EVaR}_\alpha(X)
      \;=\;
      \sup_{\beta<0}
      a_X\bigl(\beta, \alpha\bigr)
      \;\le\;
      \mathrm{VaR}_\alpha(X).
    \]
\end{proof}


\section{Details on \Cref{sec:methods}}

In this section we provide additional details on the proxy method. We explain : (1) how to derive the problemes, (2) how to optmize them using a grid-based method, and (3) how to optimize them using the optimality front.

\subsection{Derivation of Proxy Problems}
\label{app:proxy}

We first detail how the proxy problems are derived from the original problem.

\paragraph{Probability Threshold Problem.}

The inital problem is 
\[
    \min_\pi \;\Pr \!\bigl(R^\pi \;\le\; T\bigr),
\]

Using the Chernoff bound, and a few manipulations.
\begin{align*}
    \min_\pi \;\Pr\!\bigl(R^\pi \leq T\bigr) 
    &\leq \min_\pi \;\min_{\beta < 0} \mathbb{E}\bigl[e^{\beta(R^\pi - T)}\bigr] \\
    &= \min_{\beta < 0} \;\min_\pi \mathbb{E}\bigl[e^{\beta(R^\pi - T)}\bigr] \\
    &= \min_{\beta < 0} \mathbb{E}\bigl[e^{\beta\left(R^{\pi^*_\beta} - T\right)}\bigr].
\end{align*}
Where the inversion between the minimum on the policy and the minimum on $\beta$ is justified by the fact that there is only a finite number of policies.

\paragraph{Value at Risk Problem.} The initial problem is
\[
    \max_\pi \;\mathrm{(C)VaR}_\alpha\bigl[R^\pi\bigr].
\]

The similar manipulations lead to
\begin{align*}
    \max_\pi \;\mathrm{(C)VaR}_\alpha\bigl[R^\pi\bigr]
    &\geq \max_\pi \;\mathrm{EVaR}_\alpha\bigl[R^\pi\bigr] \\
    &= \max_\pi \sup_{\beta < 0} \mathrm{EntRM}_{\beta}\bigl[R^\pi\bigr] + \frac{1}{\beta}\ln(\alpha) \\
    &= \sup_{\beta < 0} \max_\pi \mathrm{EntRM}_{\beta}\bigl[R^\pi\bigr] + \frac{1}{\beta}\ln(\alpha) \\
    &= \sup_{\beta < 0} \mathrm{EntRM}_{\beta}\bigl[R^{\pi^*_\beta}\bigr] + \frac{1}{\beta}\ln(\alpha) \\
\end{align*}

\subsection{Grid-Based Optimization of the Risk Parameter}

A natural first idea to optimize the EVaR in MDPs is to discretize $\beta$ on a grid. By considering well-chosen values of $\beta$, \citet{hau_entropic_2023} show that they can get an $\varepsilon$-approximation of the optimal policy for the EVaR problem with a complexity of $O\bigl(|\mathcal{S}|^2\,|\mathcal{A}|\,H\frac{\log(1/\varepsilon)}{\varepsilon^2}\bigr)$.

We prove a similar result for the Chernoff approximation of the Threshold Probability problem.

\begin{proposition}
    \label{prop:approx_chernoff}
    Let $R$ be the return of the MDP such that there exists $a<0$ and $p>0$ with the property that for any policy $\pi$, $\Pr(R \leq a)\geq p$. Then solving the proxy problem 
    with accuracy $2\log(1+\varepsilon)/H$ on $\beta$ and $\beta_{\min} = \ln(p)/a$, finds a policy $\pi$ that satisfies
    \[
    \mathbb{P}\bigl(R^{\pi}\leq 0\bigr) \;\leq\; \tilde{B}
    \quad \text{and} \quad
    B \;\leq\; \tilde{B} \;\leq\; (1+\varepsilon)\,B,
    \]
    where $B$ is the true optimal value.
    
\end{proposition}

Hence, one needs to compute
$\frac{H\,\beta_{\min}}{2\,\log(1+\varepsilon)}$
policies to obtain a value within a factor $(1+\varepsilon)$ of the optimum. Given that computing an optimal policy for EntRM involves a complexity of $O\bigl(|\mathcal{S}|^2\,|\mathcal{A}|\,H\bigr)$, the overall complexity to achieve an approximation ratio of $(1+\varepsilon)$ is 
$O\!\Bigl(\frac{H^2\,|\mathcal{S}|^2\,|\mathcal{A}|\,\beta_{\min}}{\log(1+\varepsilon)}\Bigr).$

Note that this result translates to any value of threshold other than $0$, by just translating the rewards of the MDP so that the threshold becomes equivalent to $0$.   

As expected, those bounds show that the complexity explodes as the grid is refined to obtain more accurate approximations. However, this method does not use the knowledge about the structure of the Optimality Front, and computes the same optimal policies over and over again. 

We show in \Cref{sec:experiments} that \algoname~ computes the Optimality Front orders of magnitude more efficiently than computing on a grid. Conveniently, we can use this Optimality Front to then compute those relaxed optimization problems in a more efficient way.

Consider the EntRM optimal policies $\pi_1, \ldots, \pi_K$ and the corresponding optimality intervals $I_1, \ldots, I_K$. We have
\begin{align*}
    \min_\pi \;\Pr\!\bigl(R^\pi \leq T\bigr)
    &\leq \min_{\beta < 0} \mathbb{E}\bigl[e^{\beta(R^{\pi^*_\beta} - T)}\bigr] \\
    &= \min_k \;\min_{\beta \in I_k} \mathbb{E}\bigl[e^{\beta(R^{\pi_k} - T)}\bigr].
\end{align*}
and 
\begin{align*}
    \max_\pi \;\mathrm{(C)VaR}_\alpha\bigl[R^\pi\bigr] 
    &\geq \sup_{\beta < 0} \mathrm{EntRM}_{\beta}\bigl[R^{\pi^*_\beta}\bigr] + \frac{1}{\beta}\ln(\alpha) \\
    &= \max_k \sup_{\beta \in I_k} \mathrm{EntRM}_{\beta}\bigl[R^{\pi_k}\bigr] + \frac{1}{\beta}\ln(\alpha).
\end{align*}

These problems are reduced to more simple optimization problems on small intervals that can be solved using gradient methods. Indeed, the first one, for the Threshold Probability, is concave \citep{boyd_convex_2004}, and the second one is quasiconcave\footnote{It becomes a concave problem after using the change of variable $\beta \leftarrow \frac{1}{\beta}$ and thus can still be optimized efficiently} \citep{hau_entropic_2023}.

\subsection{Proof of \Cref{prop:approx_chernoff}}

We write $M_R(\beta)=\mathbb{E}[\exp(\beta R)]$ the moment generating function of the random variable $R$.

\begin{lemma}
\label{lemma:local_bound}
 Assume that $R$ is bounded between $-H$ and $H$, then for a fixed stepsize $\epsilon>0$, for $k \in \mathbb{N}$ and $\beta \in [\epsilon k, \epsilon (k+1)]$, noting $b_k = M_R(-\epsilon k)$, we have
 $$M_R(-\beta) \geq \exp(-\epsilon H)b_k \quad \text{and} \quad M_R(-\beta) \geq \exp(\epsilon H)b_{k+1},$$
 and,
 $$\frac{M_R(-\beta)}{\min \{b_k,b_{k+1}\}} \geq \exp\bigg(\epsilon \frac{-H}{2}\bigg) \bigg(\frac{b_k}{b_{k+1}}\bigg)^{\frac{1}{2H}} \geq \exp\bigg(\epsilon \frac{-H}{2}\bigg)\;,$$
 otherwise the ratio is greater or equal to $1$.
\end{lemma}

\begin{proof}
 First we have
 \begin{align*}
  M_R(-\beta)&=\mathbb{E}[\exp(-(\beta - \epsilon k)R)\exp(-\epsilon kR)] \\
  &\geq \exp(-(\beta - \epsilon k)H)\mathbb{E}[\exp(-\epsilon kR)] \geq \exp(-\epsilon H)b_k, \\
  M_R(-\beta)&=\mathbb{E}[\exp((-\beta + \epsilon (k+1))R)\exp(-\epsilon (k+1)R)] \\
  &\geq \exp(-(\epsilon (k+1)-\beta)H)\mathbb{E}[\exp(-\epsilon (k+1)R)] \geq \exp(-\epsilon H)b_{k+1}\;.
 \end{align*}
To obtain the following inequality it is enough to consider the intersection of the functions $t \mapsto \exp(-t H)b_k$ and $t \mapsto \exp((-\epsilon -t)H)b_{k+1}$ on $t \in [0,\epsilon]$: as $M_R(-\beta)$ follows both left decreasing and right increasing constraint, the minimal value possible is on the intersection.
\end{proof}

We can now give a proof of Proposition \ref{prop:approx_chernoff}.
\begin{proof}
Let $\pi_c$ and $\beta_c$ be the policy and the $\beta$ optimizing the Chernoff bound $B$. First, we can consider that $\beta_c < \ln(1/p)/a\}$. Indeed if not we can choose $\beta_c = 0$ instead:

\begin{align*}
M_{R}(\beta_c) &= \mathbb{E}[\exp(-\beta R)\mathds{1}\{R\leq a\}]+\mathbb{E}[\exp(-\beta R)\mathds{1}\{R>a\}] \\
&\geq \exp(-\beta_c a)\mathbb{P}[R\leq a] +\exp(-\beta M)\mathbb{P}[R \geq a]\\
&\geq \exp(-\beta_c a)p \geq 1 = M_{R}(0) \quad \text{using the hypothesis on }\beta_c\;. \\
\end{align*}

Let then $k\in\mathbb{N}$ be such that $\epsilon k\leq \beta_c < \epsilon(k+1)$, and suppose  $M_{R^{\pi_c}}(-\epsilon k)  \leq M_{R^{\pi_c}}(-\epsilon (k+1))$ without loss of generality. We have
$$\tilde{B} \leq \min_\pi M_{R^\pi}(-\epsilon k) \leq M_{R^{\pi_c}}(-\epsilon k) \leq \exp \bigg( \epsilon \frac{H}{2}\bigg)M_{R^{\pi_c}}(\beta_c) \leq \exp \bigg(\epsilon \frac{H}{2}\bigg)B,$$
using \Cref{lemma:local_bound}.

For $\epsilon \gets 2\log(1\varepsilon)/H$, we have that $\tilde{B} \leq (1+\varepsilon)B$.
\end{proof}


\section{Complements on \Cref{sec:structure}}
\label{app:proof_structure}

\subsection{Optimality Front definition}
\Cref{def:optimality_front} has a some shortcomings and the optimality front is not well defined this way. A first issue is that, for a specific $\beta$, the optimal policy may not be unique. There might be different possible choices, that could lead to also different optimality intervals. For the optimality front to be well defined, we would need unicity. A second issue comes from the fact that a policy might be optimal for a unique value of $\beta$, and not an interval of length $>0$. 
To match with the definition of \emph{breakpoints}, we would like to ensure that the optimality intervals are always of length $>0$, and that two optimality intervals may only overlaps on their endpoints. For both those reasons, we introduce the following notions.

We say that two policies $\pi_1$ and $\pi_2$ are equivalent, noted $\pi_1 \sim \pi_2$ if the associated return distributions are equal : $\pi_1 \sim \pi_2 \Longleftrightarrow R^{\pi_1} \overset{d}{=} R^{\pi_2}$.
It is an \emph{equivalence relation} (in the sense of mathematical binary relations) so it is possible to consider the quotient set $\bar \Pi = \Pi/\sim$ of equivalence classes. For $\bar \pi \in \bar \Pi$, the return distribution $R^{\bar \pi}$ is well defined. Considering the equivalence classes allows to avoid the issue of having several optimal policies on any given intervals of the risk parameter, instead, all such policies are in the same equivalence class $\bar \pi$.

For $\bar \pi \in \bar \Pi$, we consider $O_{\bar \pi} = \left\{ \beta \in \mathbb R \text{ s.t. } \mathrm{EntRM}_\beta[R^{\bar \pi}] = \sup_\pi \mathrm{EntRM}_\beta[R^{\pi}]\right\}$ the set of values for which $\bar \pi$ is optimal, and $I_{\bar \pi}= \overline{\text{int } O_{\bar \pi}}$ that set with its isolated points removed.

\begin{definition}[Optimality front]
    \label{def:optimality_front_formal}
    The optimality front $\Gamma$ is defined as $\Gamma= (\bar \pi_k, I_k)_{k \in [1,K]}$ where $I_k = I_{\bar \pi_k}$ and $(\bar \pi_k)$ is the set of class of policies such that $I_k \neq \emptyset$.    
\end{definition}

With such definition, the optimality front is uniquely defined, up to permutations. Computing it consists in computing \emph{one} representative policy of each optimal class $\bar \pi_k$, and the corresponding optimality interval $I_k$. 

\subsection{Complements on \Cref{pro:finite_action_change} and definition of Breakpoints.}

We first provide a more formal statement of the result, and then give a proof of it.

\begin{proposition}
    \label{pro:finite_action_change_formal}
    The Optimality Front is unique, and finite. Also, $\forall k, I_k$ is a finite union of closed intervals in $\mathbb R \bigcup  \{  -\infty\}$ with no isolated points, and $\bigcup_k I_k = \mathbb R^- \bigcup \{-\infty\}$. Finally, for $k_1 \neq k_2$, $I_{k_1} \cap I_{k_2} = \partial I_{k_1} \cap \partial I_{k_2} $.  
\end{proposition}

\begin{proof}[Proof of \Cref{pro:finite_action_change}]

    Let us first show that the optimality front is well defined.

    Let $\bar \pi \in \bar \Pi$. $\forall \pi_1, \pi_2 \in \bar \pi$, the distributions of the return are equal by definition: $R^{\pi_1} \overset{d}{=} R^{\pi_2}$. Hence, 
    $R^{\bar \pi}$ is well defined. Similarly 

    \begin{align*}
        O_{\pi_1} 
        &= \left\{ \beta \in \mathbb R \text{ s.t. } \mathrm{EntRM}_\beta[R^{ \pi_1}] = \sup_\pi \mathrm{EntRM}_\beta[R^{\pi}]\right\} \\
        &= \left\{ \beta \in \mathbb R \text{ s.t. } \mathrm{EntRM}_\beta[R^{\pi_2}] = \sup_\pi \mathrm{EntRM}_\beta[R^{\pi}]\right\} = O_{\pi_2}
    \end{align*}

    and so $O_{\bar \pi}$ (and consequently, $I_{\bar \pi}$) is well defined. Finally, the Optimality Front $\left\{ (\bar \pi, I_{\bar \pi}) \mid \bar \pi \in \bar \Pi, I_ {\bar \pi} \neq \emptyset \right\}$ is uniquely defined.

    We now proceed to show the intervals part. In order to do so, we introduce the following result about analytic functions.
    \begin{lemma}
        Let $f_1$ and $f_2$ be two analytic functions on a finite interval $I \subset \mathbb{R}$. If $f_1$ and $f_2$ are equal on $I$, then they are equal everywhere. If they are not equal on $I$, then they can only intersect at a finite number of points in $I$.
    \end{lemma}

    We also recall that $\bar \Pi$ is finite since the number of markov deterministic policies is finite in our finite horizon setting, and that optimizing on $\mathrm{EntRM}_\beta[R^\pi]$ is equivalent to optimizing on $-\mathbb E\left[\exp(\beta R^\pi)\right]$

    Consider $f_{\bar \pi}(\beta) = -\mathbb E\left[\exp(\beta R^{\bar \pi})\right]$. Those functions are analytic as a convex combination of exponential functions, which are analytic. 

    Assume that there exist $\bar \pi_1 \neq \bar \pi_2$ such that $f_{\bar \pi_1}(\beta) = f_{\bar \pi_2}(\beta)$ for all $\beta$. The moment generating function defines uniquely the distribution of a random variable, so it implies that $R^{\bar \pi_1} \overset{d}{=} R^{\bar \pi_2}$. This contradicts the fact that $\bar \pi_1$ and $\bar \pi_2$ are different policies. Hence, the functions $f_{\bar \pi}$ are distinct pairwise.
    
    Using \Cref{thm:beta_min}, we can restrict ourselves to only considering the risk parameter in the interval $[\beta_{\min}, 0]$. The functions $f_{\bar \pi}$ are analytic on this interval, and thus they can only intersect at a finite number of points.
    This implies that the set $\{ \beta \in \mathbb R^- \mid \exists \bar \pi_1  \neq \bar \pi_2 \in \bar \Pi, f_{\bar \pi_1}(\beta) = f_{\bar \pi_2}(\beta) \}$ is finite. We write $\beta_0 < \beta_1 < \ldots < \beta_K$ the ordered set of those points. By definition, on each interval $]\beta_k, \beta_{k+1}[$, there is a unique $\bar \pi_k$ such that $f_{\bar \pi_k}(\beta) = \sup_{\bar \pi} f_{\bar \pi}(\beta)$, and thus that $\mathrm{EntRM}_\beta[R^{\bar \pi_k}] = \sup_{\bar \pi} \mathrm{EntRM}_\beta[R^{\bar \pi}]$. Also, by continuity, both $\pi_k$ and $\pi_{k+1}$ are optimal for $\beta_{k+1}$. 
    Hence, consider $\bar \pi \in \bar \Pi$ and $k_1, \ldots, k_m$ the set of indices such that $\forall i, \bar \pi_{k_i} = \bar \pi$. Then, we have $I_{\bar \pi} \supset \bigcup_{i=1}^m [\beta_k, \beta_{k+1}]$. Furthermore, the only other values of $\beta$ for which $\bar \pi$ can be optimal, are isolated points among $\beta_0, \ldots, \beta_K$. 
    By definition, the intervals $I_{\bar \pi}$ do not include isolated points, so we have the equality $I_{\bar \pi} = \bigcup_{i=1}^m [\beta_k, \beta_{k+1}]$. 
    Hence, the intervals $I_{\bar \pi}$ are finite unions of closed intervals in $\mathbb R^-$.

    By construction, the intervals $I_{\bar \pi}$ are disjoint except for their endpoints, and the union of all the intervals is $\bigcup_{\bar \pi} I_{\bar \pi} = \mathbb R^-$, which concludes the proof of the proposition. 
\end{proof}

Using this result, we can now formally define the breakpoints of the optimality front. 

\begin{definition}
    \label{def:breakpoints_formal}
    The breakpoints of the optimality front are defined as the finite set of points $\bigcup_{\bar \pi \in \bar \Pi} \partial I_{\bar \pi}$.
    The breakpoints are the points where the optimality intervals $I_{\bar \pi_k}$ meet. 
    We denote by $\mathcal B = \{\beta_1, \ldots, \beta_K\}$ the set of breakpoints.
\end{definition}

\subsection{Proof of \Cref{thm:interval_policy_change}.}

We first start by deriving bounds on the growth of the EntRM when the risk parameter is changed slightly.

\begin{proposition}
    \label{pro:framing}
    Let $X$ be a random variable, $\beta \in \mathbb{R}^-$, and $0 < \varepsilon < |\beta|$. We assume $X$ is bounded in $[r_{\min}, r_{\max}]$. Then:

    \begin{align}
        \mathrm{EntRM}_\beta[X] 
        &\leq \mathrm{EntRM}_{\beta+\varepsilon}[X] 
        \leq \frac{\beta}{\beta+\varepsilon} \mathrm{EntRM}_\beta[X] + \frac{\varepsilon}{\beta+\varepsilon} r_{\min}, \\
        \frac{\beta}{\beta-\varepsilon} \mathrm{EntRM}_\beta[X] - \frac{\varepsilon}{\beta-\varepsilon} r_{\min} 
        &\leq \mathrm{EntRM}_{\beta-\varepsilon}[X] 
        \leq  \mathrm{EntRM}_\beta[X] 
    \end{align}
\end{proposition}

\begin{proof}
    In the following, $0 < \varepsilon < |\beta|$. We write $X = \sum \mu_i \delta_{x_i}$, with $r_{\min} = \min_i x_i$.

    \begin{align*}
        \mathrm{EntRM}_{\beta+\varepsilon}[X] &= \frac{1}{\beta+\varepsilon} \ln\left(\sum \mu_i e^{(\beta+\varepsilon) x_i}\right)\\ 
        &= \frac{1}{\beta+\varepsilon} \ln\left(\sum \mu_i e^{\beta x_i} e^{\varepsilon x_i}\right) \\
        &\leq \frac{1}{\beta+\varepsilon} \ln\left(e^{\varepsilon r_{\min}} \sum \mu_i e^{\beta x_i}\right) \quad \text{(since } \frac{1}{\beta+\varepsilon} < 0) \\
        &= \frac{1}{\beta+\varepsilon} \ln\left(\sum \mu_i e^{\beta x_i}\right) + \frac{\varepsilon}{\beta+\varepsilon} r_{\min} \\
        &= \frac{\beta}{\beta+\varepsilon} \mathrm{EntRM}_\beta[X] + \frac{\varepsilon}{\beta+\varepsilon} r_{\min},
    \end{align*}
    and
    \begin{align*}
        \mathrm{EntRM}_{\beta-\varepsilon}[X] &= \frac{1}{\beta-\varepsilon} \ln\left(\sum \mu_i e^{(\beta-\varepsilon) x_i}\right)\\ 
        &= \frac{1}{\beta-\varepsilon} \ln\left(\sum \mu_i e^{\beta x_i} e^{-\varepsilon x_i}\right) \\
        &\geq \frac{1}{\beta-\varepsilon} \ln\left(e^{-\varepsilon r_{\min}} \sum \mu_i e^{\beta x_i}\right) \quad \text{(since } \frac{1}{\beta-\varepsilon} < 0) \\
        &= \frac{1}{\beta-\varepsilon} \ln\left(\sum \mu_i e^{\beta x_i}\right) - \frac{\varepsilon}{\beta-\varepsilon} r_{\min} \\
        &= \frac{\beta}{\beta-\varepsilon} \mathrm{EntRM}_\beta[X] - \frac{\varepsilon}{\beta-\varepsilon} r_{\min}.
    \end{align*}
    The other sides of the inequalities are obtained using the monotony of the function $\beta \mapsto \mathrm{EntRM}_\beta[X]$ \citep{ahmadi-javid_entropic_2012}. We assumed $X$ to be a discrete random variable, but the proof can be extended to continuous random variables by using the Lebesgue integral instead of the sum.
\end{proof}

The next theorem is a special case of \Cref{thm:interval_action_change} for the case of a single state. While not necessary to prove the main result, the result is used in the algorithm \emph{FindBreaks}, see \Cref{app:single_state}.

\begin{theorem}[Interval of Action Optimality]
    \label{thm:interval_action_change}
    Let $r_{\min}$ (resp.\ $r_{\max}$) be the minimum (resp.\ maximum) achievable reward, and $\Delta R = r_{\max} - r_{\min}$. Suppose we have actions $(a_{(i)})_i$ ordered so that
    \[
        U^1_\beta = \mathrm{EntRM}_\beta\bigl[R(a_{(1)})\bigr] 
        \;>\;
        U^2_\beta = \mathrm{EntRM}_\beta\bigl[R(a_{(2)})\bigr] 
        \;\ge\; \dots 
        \;\ge\;
        U^n_\beta = \mathrm{EntRM}_\beta\bigl[R(a_{(n)})\bigr].
    \]
    In particular, \(a_{(1)}\) is the unique optimal action and \(a_{(2)}\) is the second-best. Define $\Delta U = U^1_\beta - U^2_\beta$.
    Then:
    \begin{itemize}
        \item If $\beta \neq 0$, for all 
        \(\beta' \in \left[\beta\left(1 - \frac{\Delta U}{U_2 - r_{\min}}\right), \beta\left(1 + \frac{\Delta U}{U_1 - r_{\min}}\right)\right]\)
        and for all $i \geq 2$, we have
        \[
            \mathrm{EntRM}_{\beta'}\bigl[R(a_{(1)})\bigr] 
            \;>\; 
            \mathrm{EntRM}_{\beta'}\bigl[R(a_{(i)})\bigr].
        \]
        \item If $\beta = 0$, then for all 
        \(\beta' \in \Bigl[-\tfrac{8\,\Delta U}{\Delta R^2},\;\tfrac{8\,\Delta U}{\Delta R^2}\Bigr]\)
        and all $i \geq 2$, we have
        \[
            \mathrm{EntRM}_{\beta'}\bigl[R(a_{(1)})\bigr]
            \;>\;
            \mathrm{EntRM}_{\beta'}\bigl[R(a_{(i)})\bigr].
        \]
    \end{itemize}
    In particular, the action $a_{(1)}$ remains strictly optimal for all $\beta'$ in the specified range.
\end{theorem}

\begin{proof}[Proof of \Cref{thm:interval_action_change}]

    Assume $\beta \neq 0$. Let

    By hypothesis, $U_\beta^1 > U_\beta^2$. We aim to show that if $\beta'$ remains within the specified range around $\beta$, action $a_{(1)}$ remains strictly optimal for $\mathrm{EntRM}_{\beta'}$.
    
    \textbf{Case $\beta' > \beta, \beta \neq 0$}

    Write $\beta' = \beta + \varepsilon$ with $\varepsilon > 0$. Assume
    \[
        \varepsilon <  \beta \,\frac{\Delta U}{r_{\min} - U_\beta^1}.
    \]
    From the previously established bounds (see \Cref{pro:framing}), we know
    \[
        U_\beta^1 
        \;\le\;
        \mathrm{EntRM}_{\beta + \varepsilon}\bigl[R(a_{(1)})\bigr]
        \;=\;
        U_{\beta'}^1,
    \]
    and
    \[
        \mathrm{EntRM}_{\beta + \varepsilon}\bigl[R(a_{(2)})\bigr]
        \;=\;
        U_{\beta'}^2
        \;\le\;
        \frac{\beta}{\beta + \varepsilon}\,U_\beta^2 
        \;+\; 
        \frac{\varepsilon}{\beta + \varepsilon}\,r_{\min}.
    \]
    Thus, showing 
    \[
        U_\beta^1
        \;>\;
        \frac{\beta}{\beta + \varepsilon}\,U_\beta^2
        \;+\;
        \frac{\varepsilon}{\beta + \varepsilon}\,r_{\min}
    \]
    implies $U_{\beta'}^1 > U_{\beta'}^2$. Rewriting, we get
    \[
        \beta\bigl(U_\beta^1 - U_\beta^2\bigr)
        \;<\;
        \varepsilon\,\bigl(r_{\min} - U_\beta^1\bigr),
    \]
    i.e.
    \[
        \beta \,\frac{\Delta U}{r_{\min} - U_\beta^1}
        \;>\;
        \varepsilon.
    \]
    The changes of sign in the inequality are due to the fact that $\beta+\varepsilon < 0$ and $U_\beta^1 > r_{\min}$.

    This is exactly our assumption on $\varepsilon$. Hence $a_{(1)}$ remains strictly better than $a_{(2)}$ at $\beta' = \beta + \varepsilon$, and by extension, better than all other actions.
    
    \textbf{Case $\beta' < \beta, \beta \neq 0$.}
    
    The similar argument holds using the two inequalities 
    \[
        U_\beta^2
        \;\ge\;
        \mathrm{EntRM}_{\beta - \varepsilon}\bigl[R(a_{(2)})\bigr]
        \;=\;
        U_{\beta'}^2,
    \]
    and
    \[
        \mathrm{EntRM}_{\beta - \varepsilon}\bigl[R(a_{(1)})\bigr]
        \;=\;
        U_{\beta'}^1
        \;\ge\;
        \frac{\beta}{\beta - \varepsilon}\,U_\beta^1
        \;-\; 
        \frac{\varepsilon}{\beta - \varepsilon}\,r_{\min}.
    \]

    \textbf{Case $\beta = 0$.}

This second part of \Cref{thm:interval_action_change}, when $\beta = 0$, uses \emph{Hoeffding's lemma} (see e.g. \citet{massart2007concentration}):
\begin{equation}
    \forall \lambda \in \RR, \quad \mathbb E[\exp(\lambda X)] \leq \exp\left(\lambda \E[X] + \frac{\lambda^2 \Delta R^2}{8}\right).
\end{equation}

    Hoeffding's lemma gives
    \begin{equation}
        \mathrm{EntRM}_\beta[X] \geq \E[X] - \beta \frac{\Delta R^2}{8}.
    \end{equation}
    We can then proceed similarly as the previous proof. Consider $0 > \beta' > \frac{8\Delta U}{\Delta R^2}$. Using the previous equation and the fact that $U_\beta^1 \geq E[R(a_1)]$, we have
\[
        U_{\beta'}^1 - U_{\beta'}^2 \geq \E[R(a_1)] - \E[R(a_2)] - \beta' \frac{\Delta R^2}{8} 
        \geq \Delta U -  \frac{8\Delta U}{\Delta R^2}\frac{\Delta R^2}{8} = 0\;,
    \]
    hence $U_\beta^1 > U_\beta^2$.

    This concludes the proof of \Cref{thm:interval_action_change}.
\end{proof}

Finally, we can prove \Cref{thm:interval_policy_change}. We use the same principle as in the proof of \Cref{thm:interval_action_change} and the principle of optimality (i.e. the optimal policy is always greedy with respect to itself) to show that the optimal policy remains the same. The inequalities from \Cref{pro:framing} are adapted by replacing $r_{\max}$ by $H-h$ and $r_{\min}$ by $h-H$ (for timestep $h$, the return is in $[h-H, H-h]$ as the reward is bounded by $[-1,1]$ at each step). The value $r_{\min} - U_\beta^1$ is replaced by its upper bound $2(H-h)$ to obtain a symetric bound and simplify the formula.

\begin{proof}{(of \Cref{thm:interval_policy_change})}
    Let us address the case where $\beta < 0$ and $\varepsilon > 0$.  
    By the principle of optimality, we have:
    \[
        \forall h, x, \quad 
        \pi_{h,\beta}^{*}(x) = \arg \max_a \mathrm{EntRM}_\beta[R^{\pi^*}_h(x,a)]\;.
    \]
    Let $\varepsilon < \Delta$ (with $\varepsilon < |\beta|$).  
    We then obtain the inequalities:
    \begin{align*}
        \forall h, x, a, \quad 
        &\mathrm{EntRM}_\beta[R^{\pi^*}_h(x,a)] 
        \;\leq\; \mathrm{EntRM}_{\beta+\varepsilon}[R^{\pi^*}_h(x,a)]\\
        \text{and} \quad 
        &\mathrm{EntRM}_{\beta+\varepsilon}[R^{\pi^*}_h(x,a)]
        \;\leq\; \frac{\beta}{\beta + \varepsilon} \mathrm{EntRM}_\beta[R^{\pi^*}_h(x,a)] 
        + \frac{\varepsilon}{\beta+\varepsilon} (H-h)\;.    
    \end{align*}

    By the definition of $\varepsilon$, and following the same method as for \Cref{thm:interval_action_change} we get $\forall h, x, \forall a\neq\pi^*_h(x)$,
    \[
        \frac{\beta}{\beta + \varepsilon} \mathrm{EntRM}_\beta[R^{\pi^*}_h(x,a)] 
        + \frac{\varepsilon}{\beta + \varepsilon} (H-H) 
        \;\leq\; \mathrm{EntRM}_{\beta}[R^{\pi^*}_h(x,\pi^*_h(x))] 
        .
    \]
    Therefore, 
    \begin{align*}
        \mathrm{EntRM}_{\beta'}[R^{\pi^*}_h(x,a)] 
        &\leq \frac{\beta}{\beta + \varepsilon} \mathrm{EntRM}_{\beta}[R^{\pi^*}_h(x,a)] 
        + \frac{\varepsilon}{\beta + \varepsilon}(H-h) \\
        &\leq  \mathrm{EntRM}_{\beta}[R^{\pi^*}_h(x,\pi^*_h(x))] \\
        & \leq \mathrm{EntRM}_{\beta'}[R^{\pi^*}_h(x,\pi^*_h(x))] \;.
    \end{align*}
    Thus, $\pi^*$ remains greedy with respect to itself and remains the optimal policy.

    The cases $\beta = 0$ and $\varepsilon < 0$ follow similarly, using the associated inequalities.
\end{proof}

\subsection{About \Cref{pro:policy_change}.}
\label{app:policy_change}
We first start with a lemma.

\begin{lemma}
    \label{lem:conditional}
    Let $Y$ be a random variable with continuous law.
    Let $X_1, \ldots, X_n$ be random variables with $Y$ independant from $(X_1, \ldots, X_n)$. Then,
    \[\Pr\bigl( Y = f(X_1, \dots, X_n) \mid X_1, \dots, X_n \bigr) = 0\]
\end{lemma}
\begin{proof}
    This is a special case of Exercise 2.1.5 in \citet{durrett2019probability}. It comes down to writing the definition of conditional probabilities and computing the integrals with Fubini's theorem.
\end{proof}

\begin{proof}(of \Cref{pro:policy_change})

    Let $\pi$ be a policy. Assume the probability transitions are fixed and only the rewards are random. We consider, for $x,h,a,a' \in \states \times [H] \times \actions \times \actions$ the set $\breakpoints^{x,h}_{a,a'}$ of parameters $\beta$ such that $\mathrm{EntRM}_{\beta}[R^{\pi}_h(x,a)] = \mathrm{EntRM}_{\beta}[R^{\pi}_h(x,a')]$. We aim to show that the sets $(\breakpoints^{x,h}_{a,a'})_{x,h,a,a'}$ have no element in common pairwise, with probability one.

    Consider orders on both $\states$ and $\actions$, and consider the associated lexigocraphic order of $[ H, 1 ] \times \states \times \actions$. We proceed by induction on this order.

    \textbf{Case $t = H$.}

    Let $x_1, x_2 \in \states$, $a_1,a'_1, a_2, a'_2 \in \actions$, with $(x_1, a_1, a_1') \neq (x_2, a_2, a_2')$ (the triplets are different, but some elements of the triplets can be equal). Consider known $\breakpoints^{x_1,H}_{a_1,a'_1}$. 

    \begin{align*}
        \forall \beta \in \mathbb{R}, \quad 
        \Pr\Big( \mathrm{EntRM}_{\beta}&[R^{\pi}_H(x_2,a_2)] = \mathrm{EntRM}_{\beta}[R^{\pi}_H(x_2,a'_2)] 
        \;\Big|\; r_H(x_1, a_1), r_H(x_1, a_1') \Big) \\
        &= \Pr\Big( r_H(x_2,a_2) = r_H(x_2,a'_2) 
        \;\Big|\; r_H(x_1, a_1), r_H(x_1, a_1') \Big) \\
        &= 0.
    \end{align*}    
 
    Because $r_H(x_2,a_2)$ and $r_H(x_2,a'_2)$ are continuous random variables and at least one of the two is not conditioned on, the probability that they are equal is zero according to \Cref{lem:conditional}.

    It is in particular true for all $\beta \in \breakpoints^{x_1,H}_{a_1,a'_1}$. Using the union bound, we get that $\breakpoints^{x_1,H}_{a_1,a'_1}$ and $\breakpoints^{x_2,H}_{a_2,a'_2}$ have no element in common with probability one.

    By considering elements in order and conditioning on the previously \emph{observed} rewards, the induction is verified for $t = H$.
    
    \textbf{Case $t < H$.}

    Let $u = h_0, x_0, a_0$. Consider all breakpoints observed before 
    \[\breakpoints = \bigcup_{\substack{(h,x, a) < u\\ (h,x, a') \leq u}}\breakpoints^{x,h}_{a,a'}. \]
    By induction, all of them are disjoint pairwise with probability one.

    \begin{align*}
        \forall \beta &\in \breakpoints, \quad 
        \Pr\Big( \mathrm{EntRM}_{\beta}[R^{\pi}_{h_0}(x_0,a_0)] = 
        \mathrm{EntRM}_{\beta}[R^{\pi}_{h_0}(x_0,a'_0)] 
        \;\Big| \big(r_h(x, a)\big)_{(h,x,a) \leq u} \Big) \\
        &= \Pr\Big( r_{h_0}(x_0,a_0) + \mathrm{EntRM}_{\beta}[R^{\pi}_{h_0+1}(X_0)] = r_{h_0}(x_0,a'_0) + \mathrm{EntRM}_{\beta}[R^{\pi}_{h_0+1}(X_0')] 
        \;\Big| \; \big(r_h(x, a)\big)_{(h,x,a) \leq u} \Big) \\
        &= \Pr\Big( r_{h_0}(x_0,a_0) = f\left(\big(r_h(x, a)\big)_{(h,x,a) \leq u}\right)
        \;\Big| \; \big(r_h(x, a)\big)_{(h,x,a) \leq u} \Big) \\
        &= 0.
    \end{align*}
    
    Where $f\left(\big(r_h(x, a)\big)_{(h,x,a) \leq u}\right) = r_{h_0}(x_0,a'_0) + \mathrm{EntRM}_{\beta}[R^{\pi}_{h_0+1}(X_0')] - \mathrm{EntRM}_{\beta}[R^{\pi}_{h_0+1}(X_0)]$ and using \Cref{lem:conditional}. The induction is then proved.

    As there is a finite number of policy, this result remains when considering the set of all policies. 

    To conclude, notice that a breakpoint $\beta_0$ is a value of risk parameter for which two different action $a_1, a_2$ have the same expected return for some state $x$ and timestep $h$. Hence, $\beta_0 \in \breakpoints^{x,h}_{a_1,a_2}$. With probability one, those set are pairwise disjoint, meaning a single action changes in $\beta_0$.

    As this proof works for any value of the probability transitions, the result also remains for $p$ random.
\end{proof}

The result is what we observe for all the MDPs used in practice. Yet it is simple to construct an MDP where several actions change at the same time, it is enough to have two states $x_1, x_2$ have the same transitions function and rewards : $\forall a, x', p(x' | x_1, a) = p(x' | x_2, a)$ and $r(x_1, a) = r(x_2, a)$.

\subsection{Proof of \Cref{pro:eq_breakpoint}.}
\begin{proof}(of \Cref{pro:eq_breakpoint})
    It all relies on the continuity of the function $\beta \rightarrow \mathrm{EntRM}_\beta[R^\pi]$.  

    \begin{align*}
        \forall h,x, \quad 
        &\mathrm{EntRM}_{\beta_1}[R^{\pi^1}_h(x)] \geq \mathrm{EntRM}_{\beta_1}[R^{\pi^2}_h(x)]\\
        &\mathrm{EntRM}_{\beta_3}[R^{\pi^1}_h(x)] \leq \mathrm{EntRM}_{\beta_3}[R^{\pi^2}_h(x)]
    \end{align*}

    By continuity, there exists $\beta_2$ such that there is an equality, and, since we assume that there are no other optimal policy in $[\beta_1, \beta_3]$, the equality is verified for all $x,h$.
\end{proof}

\subsection{On the lowest breakpoint.}
\begin{theorem}
    \label{thm:beta_min}
    Consider $\pi_{\inf} = \lim_{\beta \rightarrow -\infty} \pi^*_\beta$. Consider the return $R$ taking value in $x_1 \leq \ldots \leq x_n$. We write $a_i(\pi) = P(R^{\pi_{\inf}} = x_i) - \Pr(R^{\pi}  = x_i)$. Then, the lowest breakpoint $\beta_{\inf}$ verifies
    \[
        \beta_{\inf} \geq \frac{-\log\left(1 + \max_{\pi \neq \pi_{\inf}} \max_{i > 0} \frac{|a_i(\pi)|}{|a_0(\pi)|}\right)}{\min_{i\neq j} |x_i - x_j|}\ .
    \]
\end{theorem}
\begin{proof}(of \Cref{thm:beta_min})
    
    We write $\Delta R = R_{\max} - R_{\min}$. We assume here for simplicity that all the rewards are evenly spaced on a grid. This mean that the values of the return are of the form $R_i = R_{\min} + i \frac{\Delta R}{n}$ for $i \in \{0, \ldots, n\}$.

With this natural assumption, the problem of finding $\beta$ such that,
\[\mathrm{EntRM}_\beta[R^\pi] = \mathrm{EntRM}_\beta[R^{\pi'}]\]
can be transformed into a problem of finding the roots of a polynomial.

\begin{align*}
    \mathrm{EntRM}_\beta\bigl[R^\pi\bigr] \;=\; \mathrm{EntRM}_\beta\bigl[R^{\pi'}\bigr]
    &\;\Longrightarrow\;
    \mathbb{E}\bigl[\exp(\beta R^\pi)\bigr] 
    \;=\;
    \mathbb{E}\bigl[\exp(\beta R^{\pi'})\bigr]
    \\
    &\;\Longleftrightarrow\;
    \sum_{i=0}^n \mu_i \,\exp\!\bigl(\beta R_i\bigr)
    \;=\;
    \sum_{i=0}^n \mu'_i \,\exp\!\bigl(\beta R_i'\bigr)
    \\
    &\;\Longleftrightarrow\;
    \sum_{i=0}^n a_i \,\exp\!\bigl(\beta R_i\bigr)
    \;=\;
    0
    \\
    &\;\Longleftrightarrow\;
    \sum_{i=0}^n a_i \,\exp\!\Bigl(\beta\,i\,\frac{\Delta R}{n}\Bigr)
    \;=\;
    0
    \\
    &\;\Longleftrightarrow\;
    \sum_{i=0}^n a_i \,X^{-i}
    \;=\;
    0
    \,.
    \\
    &\;\Longleftrightarrow\;
    \sum_{i=0}^n a_{n-i} \,X^{i}
    \;=\;
    0
    \,.
\end{align*}
Where $X = \exp\!\Bigl(-\beta\,\frac{\Delta R}{n}\Bigr)$. The first implication is not an equivalence because of the case $\beta = 0$, where the exponential form (rhs) is always equal to $1$ and thus the equality is always verified. The last implication is verified by multiplying by $X^n$ because of that same fact that $0$ is already a root of the equation.

Hence, if $X$ is a solution, $\beta$ verifies $\beta = -\frac{\log(X)}{\frac{\Delta R}{n}}$.

Cauchy's bound on the size of the largest polynomial root \citep{cauchy1828exercices} claims that the largest root verifies
\[1 + \max_{i > 0} \frac{|a_i|}{|a_0|}.\]

The lowest breakpoints corresponds to the largest breakpoint solution between the $\pi_{\inf}$ and any other policy. Hence, the lowest breakpoint verifies

\[\beta_{\inf} \geq \frac{-\log\left(1 + \max_{\pi \neq \pi_{\inf}} \max_i \frac{|a_i(\pi)|}{|a_0(\pi)|}\right)}{\frac{\Delta R}{n}}\ .\]

\end{proof}

This proof uses Cauchy's bound on the size of the largest polynomial root as it is simple to write and an efficient bound. Tighter bound have been developed in the litterature that could also be used here. See \citet{akritas2008improving} for example.

Using this polynomial formulation, it is also possible to derive a theoretical bound on the smallest distance between breakpoints. Mignotte's separation bound \citep{collins2001polynomial} gives a lower bound on the distance between two roots of the polynomial, as a function of the coefficient of such polynomial. This bound can then be transfered to a bound on the distance between breakpoints. This kind of bound could be useful to choose the necessary precision for the computation of the breakpoints, but are intractable to compute and the obtained values are too small to be relevant in practice.

\subsection{About \Cref{pro:recursive_breakpoint}.}
\label{app:recursive_breakpoint}

\begin{proposition}[Formal]
    Let $\breakpoints^h$ and $\Gamma^h$ respectively be the set of breakpoints and the optimality front when the MDP starts at timestep $h$. Let $\breakpoints\left((X_i)_i, I\right)$ the set of breakpoints for the MDP with a single state and reward distributions $(X_i)_i$ (not assumed deterministic).
    \[
    \breakpoints^h = \breakpoints^{h+1} \cup \left( \bigcup_{x\in\states, (\pi_k^h, I_k^h) \in \Gamma^h} \breakpoints\left( [R^{\pi_k^h}_h(x,a)]_a, I_k^h\right) \right)
    \]
\end{proposition}

\begin{proof}
    Consider $\beta_0$ a breakpoint. There exists $h$ such that $\beta_0 \in \breakpoints^h \setminus \breakpoints^{h+1}$. This $h$ corresponds to the last timestep where the optimal policy changes when the risk parameter crosses $\beta_0$. 
    
    In particular, let $(\pi^I, I) \in \Gamma^{h+1}$ such that $\beta_0 \in I$. By definition, 
    \[\forall t \geq h+1, \forall\beta \in I,\quad [\pi^*_{\beta}]_t(x) = [\pi^I]_t(x).\]
    
    Also, since $\beta_0$ is a breakpoint, there exists $x, a_1 \neq a_2, \pi_1 \neq \pi_2$ such that 
    \[\mathrm{EntRM}_{\beta_0}[R^{\pi^1}_h(x,a_1)] = \mathrm{EntRM}_{\beta_0}[R^{\pi^2}_h(x,a_2)],\] 
    
    By definition of $\pi^I$, it is also equal to
    \[\mathrm{EntRM}_{\beta_0}[R^{\pi^I}_h(x,a_1)] = \mathrm{EntRM}_{\beta_0}[R^{\pi^I}_h(x,a_2)],\]

    Hence, $\beta_0 \in \breakpoints\left( [R^{\pi^I}_h(x,a)]_a, I\right)$.

\end{proof}


\section{Solving the exact breakpoints}
\label{app:exact_breakpoints}

Several issues rises when trying to computing the Optimality front with this method. The first one is that there is no easy way to compute the function $\beta \mapsto Q^\pi_{h,\beta}(x,\pi(x))$ without having access to the exact distribution of return for this policy, or to perform a Policy Evaluation step for each value of $\beta$, which quickly becomes computationally inefficient. 

A second issue is that this system of equation is only valid if there is a single breakpoint (i.e., no other optimal policy) between those two policies. Assume that we know $\pi_1$ is optimal for $\beta_1$ and $\pi_2$ is optimal for $\pi_2$, solving the equations gives value of the risk parameter for which one policy becomes better than the other. However, there could also be a third (or more) policy $\pi_3$ which is optimal for values of $\beta$ between $\beta_1$ and $\beta_2$. Computing the optimality front would require computing the breakpoints between $\pi_1$ and $\pi_3$, and between $\pi_3$ and $\pi_2$. As there are no simple conditions to verify the existence of another optimal policy between two known one, solving the exact breakpoint for a policy $\pi_1$ optimal for $\beta_1$ would require to solve the system of equation for all possible policy $\pi_2$, and retrieving the lowest breakpoint obtained among them all. Selecting this lowest breakpoint $\beta_2$ ensures that $\pi_1$ is only optimal for the a risk parameter up to $\beta_2$. Because of the exponential number of possible policies, such method is intractable in practice, and some approximation will be required.

Lastly, the equations themselves are non-trivial. Indeed, knowing the distribution of rewards, such an equation takes the form $\sum_{i=1}^{k} a_i \, e^{\alpha_i \beta} \;=\;0$, where the coefficient $a_i$ may be non-positive.
The function in question can be expressed as the difference of two convex functions, but it lacks many of the regularity properties needed for simple optimization methods. To the best of our knowledge, the best algorithms for solving such problems reduce it to finding roots of polynomial and using efficient solvers. Yet, the non-linear transformation required (see proof of \Cref{thm:beta_min}) makes the approximation error become significant when transposed back to our initial problem.


\section{Analysis of the number of breakpoints}
\label{app:breakpoints_number}

\subsection{Theoretical bound on the number of breakpoints}

\begin{proposition}
    \label{pro:max_breakpoints}
    Let $n$ the number of possible values of $R^{\pi}$. The number of breakpoints $B$ verifies
    \[ B \leq n\cdot |\actions|^{2|X|H}\]
\end{proposition}

\begin{proof}(of \Cref{pro:max_breakpoints})

    We first consider this lemma, on the number of roots of an exponential sum.
    \begin{lemma}[\citep{tossavainen2007lost}]
        \label{lem:exp_root}
        Let $f_n(x) = \sum^n_0 b_i k_i^x$, $b_i \in \RR, k_i > 0$. Then $f_n$ has at most $n-1$ roots.
    \end{lemma}

    This lemma allows for bounding the number of values of the risk parameter $\beta$ for which two different policies have the same entropic risk.
    \begin{proposition}
        \label{pro:nbr_crossing_deux}
            Consider two distributions $\mu_1 \neq \mu_2$ with support on $X = \{x_1, \dots, x_n\} \subset \RR$ of size $n$. Consider $X_1 \sim \mu_1, X_2 \sim \mu_2$ random variables following those distributions. Consider $\breakpoints = \{\beta \in \RR \ | \ \mathrm{EntRM}_\beta(X_1) = \mathrm{EntRM}_\beta(X_2)\}$. 
            
            Then:  $|\breakpoints| \leq n-1$
    \end{proposition}
    The proposition is straightforward by considering the exponential form of the EntRM, which is a sum of exponentials.

    Finally we conclude by considering that there are $|\actions|^{|X|H}$ deterministic markovian policies for a specific MDP, and thus $|\actions|^{2|X|H}$ pairs of policies. The breakpoints are values for which the entropic risk of two policies is equal, and thus are included in the union of the "breakpoints" of all pairs of policies. By \Cref{pro:nbr_crossing_deux}, the number of breakpoints is at most $(n-1)|\actions|^{2|X|H}$.
\end{proof}

This proposition hat no reason to be tight in general. In the conclusion we consider the number of point of equality between any two policies, but those points cannot all be breakpoints. Only the points where one of the two policy is optimal count as breakpoints. This problem is reduced to finding the number of components of the function $h(\beta) = \max \{ f_\pi(\beta) \}_{\pi \in \Pi}$, where $f_\pi(\beta) = \mathrm{EntRM}_\beta[R^\pi]$. This combinatorial problem has been studied before (see Lemma 2.4 in \citet{atallah1985some}), but to the best of our knowledge, no better explicit formula has been found.

\subsection{Evolution with the number of actions and atoms in the distributions}
\label{app:breakpoints_number_state}

We conducted experiments in the single-state setting (the \emph{FindBreaks} setting) to determine whether the theoretical bound on the number of breakpoints is reached in practice. We generated numerous random reward distributions and used our algorithm to find the number of breakpoints. We performed two experiments: one evaluating how the number of breakpoints evolves with respect to the number of actions, and the other with respect to the number of atoms in the reward distributions.

All the distributions considered have values in $[0,1]$ with atoms evenly spaced over this interval. To generate these distributions, we treat each as an element of $[0,1]^n$, where the sum of the elements equals 1, representing a point on the $n$-dimensional simplex. The distributions are thus generated by uniformly sampling a point on this simplex.

For the first experiment, we fixed the number of atoms to 10 and varied the number of actions from 5 to 50. For the second experiment, we fixed the number of actions to 10 and varied the number of atoms from 5 to 50 as well. For each action, the reward distribution was generated randomly as previously described. In the solving algorithm, we searched for breakpoints for $\beta$ values in the range $[-15, 15]$ with a precision of 0.01. For each plot, we generated 100 independent problems and displayed a histogram of the number of breakpoints found across these 100 problems.

\begin{figure}[ht]
    \centering
    \includegraphics[width=\textwidth]{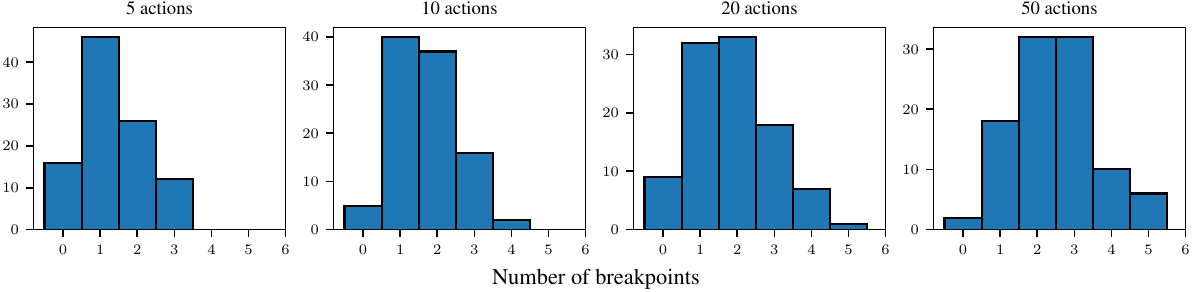}
    \includegraphics[width=\textwidth]{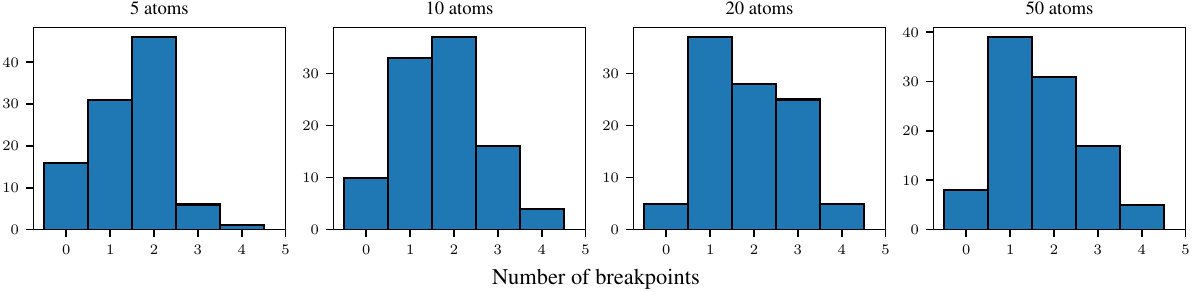}
    \vspace{-0.6cm}
    \caption{Illustration of the evolution of the number of breakpoints when the number of actions (Up) and atoms in the distributions (Down) is increasing. While the number of actions and atoms is multiplied by 10, the average number of breakpoints increases by less than a factor 2.}
    \label{fig:hist_breakpoints}
\end{figure}

The results are shown in \Cref{fig:hist_breakpoints}. We observe that the number of breakpoints increases according to the studied parameters, with an average number of crossings of 1.25, 1.9, 2.19, and 2.36 for the first experiment, and 1.42, 1.66, 1.93, and 2.05 for the second. However, this increase is far from the theoretical bound established in \Cref{pro:max_breakpoints}. Indeed, the number of breakpoints is much lower than the theoretical bound, showing sublinear growth. This experiment confirms the efficiency of our algorithm, whose complexity is strongly tied to the number of breakpoints.

\section{FindBreaks Algorithm}
\label{app:single_state}

Here \emph{FindBreaks} is designed to be able to handle values of the risk parameter both positive and negative. In practice, using DOLFIN, all the values will be non-positive. The values of the increments $\Delta$ come from \Cref{thm:interval_action_change}.

\begin{algorithm}
  \caption{FindBreaks: Computing all optimal actions}\label{alg:state}
  \begin{algorithmic}[1]
    \Require Precision $\varepsilon \in (0,1)$, Random rewards $(R(a_i))_i$, interval $I$.
    \State Compute $a^* = \arg\max_a \mathrm{EntRM}_0(R(a))$ \Comment{Initial optimal action}
    \State Compute $A = \min_{a_2 \neq a^*} \mathrm{EntRM}_0(R(a^*)) - \mathrm{EntRM}_0(R(a_2))$, $\Delta R = r_{\max} - r_{\min}$ \Comment{Advantage function and range of rewards}
    \State Initialize $\beta = \frac{8\Delta\mu}{\Delta R^2}$, $\beta_{\text{old}} = 0$ \Comment{Initial parameter values}
    \State Initialize $b_\ell = 0$, $a_{\beta_{\text{old}}} = a^*$
    \While{$\beta_{\min} < \beta < \beta_{\max}$}
      \State Compute the $\beta$ risks $\mathrm{EntRM}_\beta(R(a))$ for each action $a$ \Comment{Evaluate risks for all actions}
      \State $a_\beta = \arg \max_a \mathrm{EntRM}_\beta(R(a))$ \Comment{Select the action with the highest utility}
      \If{$a_\beta \neq a_{\beta_{\text{old}}}$}
        \State Add $[b_\ell, \beta]$ to the interval set $\mathcal{I}$ and $a_\beta$ to the optimality front $\Pi^*$ \Comment{Update intervals and front}
        \State $b_\ell \gets \beta$ \Comment{Update the lower bound of the next interval}
      \EndIf
      \State $a_{\beta_{\text{old}}} \gets a_\beta$ \Comment{Update the last optimal action}
      \State $\Delta U = \min_{a \neq a_\beta} \bigl(\mathrm{EntRM}_\beta(R(a_\beta)) - \mathrm{EntRM}_\beta(R(a))\bigr)$ \Comment{Smallest optimality gap for non-optimal actions}
      \If{$\beta < 0$}
        \State $\beta \gets \beta - \max\{\beta \frac{\Delta U}{\Delta R}, \varepsilon\}$ \Comment{Decrease $\beta$ for negative values}
      \EndIf
      \If{$\beta > 0$}
        \State $\beta \gets \beta + \max\{\beta \frac{\Delta U}{\Delta R}, \varepsilon\}$ \Comment{Increase $\beta$ for positive values}
      \EndIf
    \EndWhile
    \State Add $[b_\ell, \beta_{\max}]$ to $\mathcal{I}$ \Comment{Add the last interval to the set}
    \State \textbf{Output} Optimality Front $\Gamma^*$
  \end{algorithmic}
\end{algorithm}

\paragraph{Empirical Evaluation and Performance Analysis}

A simple simulation illustrates the behavior or \Cref{alg:state}. We consider two actions, $a_1$ and $a_2$, with reward distributions $\varrho(a_1) = \frac{1}{2}\delta_0 + \frac{1}{2}\delta_1$ and $\varrho(a_2) = \frac{99}{100}\delta_0 + \frac{1}{100}\delta_2$. Action $a_1$ is better in expectation (lower risk) but action $a_2$ can achieve a higher reward with small probability and should only outperform action $a_1$ for large risk parameters.  
\Cref{fig:intervals_illustration} illustrates this: the functions $\mathrm{EntRM}(R(a))$ are plotted for\footnote{This experiment is design to test the transition to a risky action, so it is only relevant to observe the optimality front for $\beta>0$.} $\beta>0$. As soon as the risk parameter $\beta$ is large enough (specifically, $\beta > 3.9$), action $a_2$ becomes optimal.
\begin{figure}[ht]
    \centering
    \includegraphics[width=\textwidth]{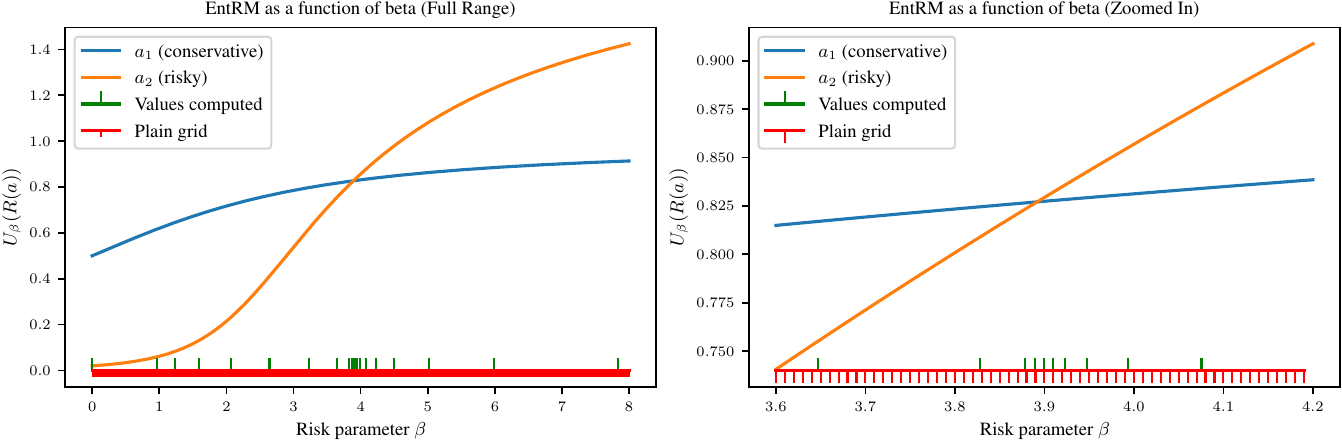} 
    \caption{Utility functions on a single-state problem: a conservative (blue) and risky (red) actions, with respective ranges of optimality intersecting at $\beta_\text{bp}=3.9\pm 10^{-2}$. In red, we show the 800 values of $\beta$ tested with a naive grid; in green the 22 values tested by \Cref{alg:state} to identify the breakpoint. Right-hand figure zooms around $\beta_\text{bp}$.
    }
    \label{fig:intervals_illustration}
\end{figure}

\Cref{alg:state} was executed on this example for $\beta \in [0,8]$, with a precision of $\varepsilon = 10^{-2}$, but our theoretical upper bound\footnote{We show a larger upper bound for visualisation purposes. Our algorithm only evaluates 3 values past this limit so it does not hurt performance significantly.} is $\beta_{\max}=\ln(100)=4.6$. The green markers correspond to the values of $\beta$ where the algorithm computed the EntRM, while the red markers represent the values that would be computed using a plain grid search with precision $\varepsilon$. As expected, we observe that the intervals shrink 
near the breakpoint, but grow significantly larger as we move away from these regions. \Cref{fig:intervals_illustration} (Right) zooms in on the interval $\beta \in [3.6, 4.2]$ to better visualize the concentration of intervals. Around $\beta = 3.9$, we observe that a few intervals are indeed capped by the maximal precision.

In this simple example, the naive grid uses $800$ evaluations while \Cref{alg:state} only requires $22$, with an \emph{efficiency ratio} of $800/22=36$. To better quantify this gain, we run another experiment on random problems with 8 actions and reward function supported on 20 atoms in $[0,1]$ (hence with possibly much more that $1$ breakpoint). On \Cref{fig:performance_algo_state}, for each level of precision $\varepsilon \in [10^{-3},10^{-1}]$, we compare the number of evaluations required by \Cref{alg:state} with the naive $1/\varepsilon$ obtained with a plain grid search. We report the gain on average over 20 random problem and notice efficiency ratios up to 35 for high-precision $\varepsilon$ values. 

\begin{figure}[ht]
    \centering
    \includegraphics[width=0.9\textwidth]{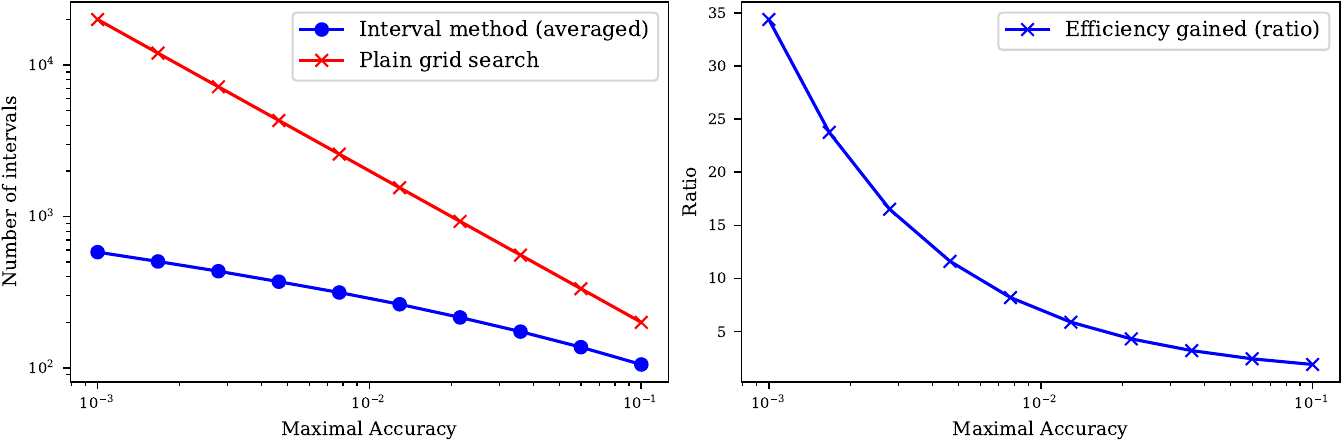}
    \caption{Performance gained using \Cref{alg:state}. (Left) Number of evaluations, (Right) efficiency ratio (red over blue).}
    \label{fig:performance_algo_state}
\end{figure}

The performance of this algorithm is also closely tied to the number of breakpoints. The more breakpoints there are, the smaller the intervals will tend to be, which in turn reduces the algorithm's efficiency compared to a plain grid search. Therefore, the number of breakpoints is a critical factor. \Cref{app:breakpoints_number} show that the number of breakpoints grows much slower than the theoretical bound. This allows our algorithm to be more efficient in practice than predicted by theory.


\section{More Numerical Experiments}
\label{app:more_experiments}

\subsection{Inventory Management}

We first consider some more values for the Inventory Management environments.

\begin{table}[htbp]
    \centering
    \caption{Evaluation of $P(R^\pi \leq T)$ for Inventory Management.}
    \label{tab:tau_results_bis}
    \begin{tabular}{lccccc}
    \toprule
    $T/\mu^*$    & 0.25  & 0.33  & 0.5  & 0.66   & 0.75      \\
    \midrule
    \textbf{Optimality Front} & $\mathbf{1.26e^{-5}}$  & $\mathbf{8.40e^{-5}}$ & $\mathbf{3.26e^{-3}}$ & $\mathbf{3.91e^{-2}}$   & $\mathbf{8.78e^{-2}}$   \\
    Proxy Optimization     & $2.33e^{-5}$  & $1.18e^{-4}$ & $3.28e^{-3}$ & $\mathbf{3.91e^{-2}}$   & $\mathbf{8.78e^{-2}}$   \\
    Risk neutral optimal & $1.11e^{-4} $  & $4.24e^{-4}$  & $5.73e^{-3}$   &$4.62e^{-2}$   & $9.77e^{-2}$        \\
    Nested Prob. Thresh. & $1.54e^{-3}$  & $8.37e^{-3}$  & $1$ & $1$  & $1$   \\
    \midrule
    Optimal value &  $6.71e^{-8}$  & $6.29e^{-7}$  & $4.48e^{-5}$  & $7.85e^{-4}$  & $3.59e^{-3}$   \\
    \bottomrule
    \end{tabular}
\end{table}

\begin{table}[htbp]
    \centering
    \caption{Evaluation of $\mathrm{(C)VaR}_\alpha[R^\pi]$ for Inventory Management.}
    \label{tab:var_results_bis}
    \begin{tabular}{l|cccc|cccc}
    \toprule Risk Measure & \multicolumn{4}{c|}{VaR} & \multicolumn{4}{|c}{CVaR} \\
    Risk parameter $\alpha$   & 0.05  & 0.1 & 0.2 & 0.5 & 0.05 & 0.1 & 0.2 & 0.5\\
    \midrule
    \textbf{Optimality Front}  & $\mathbf{1.25}$  & $\mathbf{1.33}$ & $\mathbf{1.45}$ & $\mathbf{1.65}$ & $\mathbf{1.14}$  & $\mathbf{1.21}$ & $\mathbf{1.30}$ & $\mathbf{1.45}$ \\
    Proxy Optimization & $1.22$ & $1.30$ & $\mathbf{1.45}$ & $\mathbf{1.65}$   & $1.13$ & $1.20$ & $1.30$ & $\mathbf{1.45}$ \\
    Risk neutral optimal & $1.22$ & $\mathbf{1.33}$ & $1.43$ & $\mathbf{1.65}$ & $1.11$ & $1.19$ & $1.28$ & $1.44$ \\
    Nested Risk Measure & $0.88$ & $0.95$ & $1.05$ & $1.28$ & $0.75$ & $0.84$ & $0.95$ & $1.12$ \\
    \bottomrule
    \end{tabular}
\end{table}

The values in \Cref{tab:tau_results_bis} and \Cref{tab:var_results_bis} reveal that the improvements of our \emph{Optimality Front} method over the \emph{Proxy Optimization} becomes less significant as the level of risks decreases.

We also plot in \Cref{fig:performance_algo_mdp_ratio_inv} the efficiency ratio of using the Findbreaks, similar to \Cref{fig:performance_algo_state}. This figure highlights the polynomial gain in performances of using DOLFIN and Findbreaks when one wants to compute all the optimal policies up to a certain accuracy.

\begin{figure}[ht]
    \centering
    \includegraphics[width=0.5\textwidth]{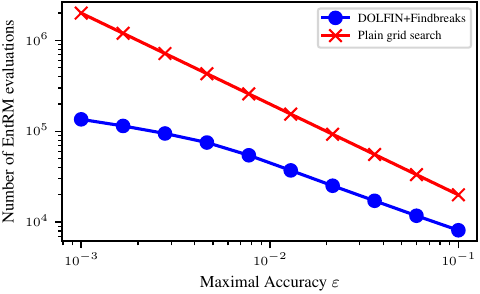}
    \caption{Performance gained using \Cref{alg:mdp} on Inventory Management.}
    \label{fig:performance_algo_mdp_ratio_inv}
\end{figure}

\subsection{Cliff Grid World}
\label{app:cliff}

Here we consider the Cliff grid world \citep{sutton2018reinforcement} illustrated in \Cref{fig:cliff_policies}.  The agent starts in the blue state. At each step, she has a small probability ($0.1$ here) of moving to another random direction. Due to these random transitions, it is risky to walk too close to the cliff (bottom, in red, negative reward $-\frac{1}{2}$), and conservative policies will prefer to walk further away to reach the goal (in green). The horizon is fixed at $H = 15$, so in principle, the agent has enough time to reach the end using the safe path. The reward when the goal is reached at step $h$ is $1 - \frac{h}{2H}$, which encourages the agent to reach it as fast as possible.

The agent thus faces a dilemma: she could either walk along the cliff, risking to fall but reaching the goal faster and consequently receiving a higher reward; or she could go all the way around, taking more time but with a lower probability to fall down. This can be observed in \Cref{fig:cliff_policies}, where the optimal policies for different values of $\beta$ are shown. For high values of beta (e.g., $\beta = 10$), the agent takes the risky path, while for low values of beta (e.g., $\beta = -5$), the agent takes the safe path.  One can even observe that for extremely negative values of $\beta$ values (e.g. $\beta =-10$) the agent prefers to stay away from the cliff, not even trying to reach the goal (see purple arrow in top right corner pointing up).

The Cliff environment is very different from Inventory Management in its nature. The agent only receives rewards when a certain state is reached, making the reward scarce and the return distribution simpler. This implies that the optimal policy for different measures of risks, such as Probability Threshold, VaR and CVaR are markovian for this MDP.

For the Threshold Probability problem, we only consider 2 values of the threshold, $-0.5$ corresponding to falling into the cliff, and $0$ corresponding to not reaching the goal. For the first threshold, the objective is to find the policy that is least likely to fall,  while for the second it is to find the policy with the most chances of reaching the goal. 

\Cref{tab:tau_results} confirms that our \emph{Optimality Front} method performs better than other methods for the Threshold probability. An important remark is that, here, the real optimal value is reached by the optimality front. Similar performances are observed for the CVaR in \Cref{tab:var_results_cliff}. For the VaR, the gain in performance is limited, which is explained by the scarcity of rewards in the environment (the small changes in the return distribution does not change the value of the VaR).

Compared to the Inventory Management, the gain performance for computing the optimality front is much better, with a ratio up to a factor $100$, as seen in \Cref{fig:performance_algo_mdp_cliff}.

\begin{table}[htbp]
    \centering
    \caption{Evaluation of $P(R^\pi \leq T)$ for Cliff.}
    \label{tab:tau_results_cliff}
    \begin{tabular}{lcc}
    \toprule
    $T$    & -0.5  & 0        \\
    \midrule
    \textbf{Optimality Front} & $\mathbf{3.72e^{-2}}$  & $\mathbf{4.65e^{-2}}$   \\
    Proxy Optimization     & $3.85e^{-2}$  & $4.67e^{-2}$  \\
    Risk neutral optimal & $4.50e^{-2}$  & $4.84e^{-2}$   \\
    Nested Prob. Thresh. & $1$  & $1$   \\
    \midrule
    Optimal value      & $\mathbf{3.72e^{-2}}$  & $\mathbf{4.65e^{-2}}$  \\
    \bottomrule
    \end{tabular}
\end{table}

\begin{table}[htbp]
    \centering
    \caption{Evaluation of $\mathrm{(C)VaR}_\alpha[R^\pi]$ for Cliff.}
    \label{tab:var_results_cliff}
    \begin{tabular}{l|cccc|cccc}
    \toprule
    Risk Measure &  \multicolumn{4}{c|}{VaR} & \multicolumn{4}{|c}{CVaR} \\
    Risk parameter $\alpha$ & 0.05  & 0.1 & 0.2 & 0.5 & 0.05 & 0.1 & 0.2 & 0.5\\
    \midrule
    \textbf{Optimality Front}  & $\mathbf{0.53}$  & $\mathbf{0.63}$ & $\mathbf{0.70}$ & $\mathbf{0.76}$ & $\mathbf{-0.37}$ & $\mathbf{0.11}$ & $\mathbf{0.38}$ & $\mathbf{0.58}$ \\
    Proxy Optimization & ${0.00}$ & ${0.6}$ & $\mathbf{0.70}$  & $\mathbf{0.76}$ & $-0.39$ & $-0.08$ & $0.38$ & $0.58$ \\
    Risk neutral optimal & $\mathbf{0.53}$ & $\mathbf{0.63}$ & $\mathbf{0.70}$ & $\mathbf{0.76}$ & $-0.43$ & $0.08$ & $0.37$ & $0.58$ \\
    Nested Risk Measure  & $-0.5$ & $-0.5$ & $-0.5$  & $-0.5$ & $-0.5$ & $-0.5$ & $-0.5$ & $-0.5$\\
    \bottomrule
    \end{tabular}
\end{table}

\begin{figure}[ht]
    \centering
    \includegraphics[width=0.9\textwidth]{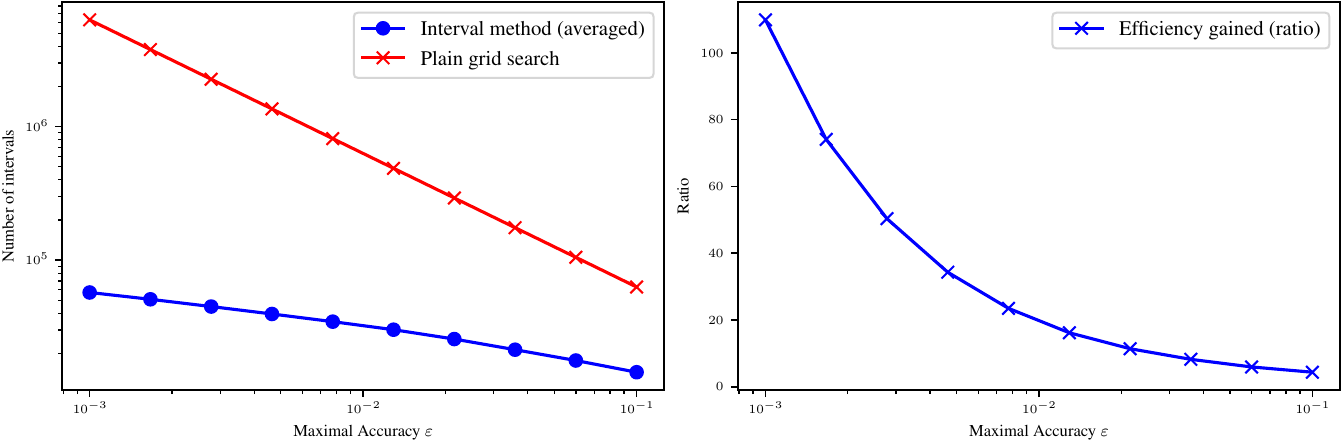}
    \caption{Performance gained using \Cref{alg:mdp} on Cliff. (Left) Number of evaluations, (Right) efficiency ratio (red over blue).}
    \label{fig:performance_algo_mdp_cliff}
\end{figure}